%% file: main.tex
\documentclass{article}

\usepackage{microtype}
\usepackage{graphicx}
\usepackage{subfigure}
\usepackage{booktabs} 

\usepackage{hyperref}

\usepackage[accepted]{icml2025}

\usepackage{amsmath}
\usepackage{amssymb}
\usepackage{mathtools}
\usepackage{amsthm}

\usepackage{booktabs}       
\usepackage{amsfonts}       
\usepackage{xcolor}         

\usepackage{colortbl}
\usepackage{bbding}
\usepackage{makecell}
\usepackage{multirow}
\usepackage{graphicx}
\usepackage{latexsym}
\usepackage{enumitem}

\usepackage[capitalize,noabbrev]{cleveref}

\theoremstyle{plain}
\newtheorem{theorem}{Theorem}[section]
\newtheorem{proposition}[theorem]{Proposition}
\newtheorem{lemma}[theorem]{Lemma}
\newtheorem{corollary}[theorem]{Corollary}
\theoremstyle{definition}

\newtheorem{assumption}[theorem]{Assumption}
\theoremstyle{remark}

\usepackage[textsize=tiny]{todonotes}

\icmltitlerunning{Does Low Rank Adaptation Lead to Lower Robustness
  against Training-Time Attacks?}

\begin{document}

\twocolumn[
\icmltitle{Does Low Rank Adaptation Lead to Lower Robustness\\
  against Training-Time Attacks?}



\icmlsetsymbol{equal}{*}

\begin{icmlauthorlist}
\icmlauthor{Zi Liang}{polyu}
\icmlauthor{Haibo Hu}{polyu}
\icmlauthor{Qingqing Ye}{polyu}
\icmlauthor{Yaxin Xiao}{polyu}
\icmlauthor{Ronghua Li}{polyu}
\end{icmlauthorlist}

\icmlaffiliation{polyu}{The Hong Kong Polytechnic University, Hong
  Kong, China}

\icmlcorrespondingauthor{Haibo Hu}{haibo.hu@polyu.edu.hk}

\icmlkeywords{Large Language Models, Low Rank Adaptation, Robustness}

\vskip 0.3in
]



\printAffiliationsAndNotice{}  

\begin{abstract}

Low rank adaptation (LoRA) has emerged as a prominent technique for
fine-tuning large language models (LLMs) thanks to its superb 
efficiency gains over previous methods.
While extensive studies have examined the performance and structural
properties of LoRA, its behavior upon training-time attacks remain underexplored, posing significant
security risks.
In this paper, we theoretically investigate the security
implications of LoRA's low-rank structure during fine-tuning,
in the context of its robustness against data poisoning
and backdoor attacks.
We propose an analytical framework that models LoRA’s training dynamics, employs the neural tangent kernel to simplify the analysis
of the training process, and applies information theory to
establish connections between LoRA's low rank structure and its vulnerability
against training-time attacks.
Our analysis indicates that \textbf{LoRA exhibits better robustness to
backdoor attacks than full fine-tuning, while becomes more
vulnerable to untargeted data poisoning} due to its over-simplified
information geometry.
Extensive experimental evaluations have corroborated our theoretical
findings.
\end{abstract}

\input{intro}

\input{method}

\input{eval}

\section{Conclusion}
This paper explores the potential training-time security risks of
LoRA-based fine-tuning. Based on the definition of training-time
robustness, this paper constructs and compares the neural tangent kernels and
the information geometry of LoRA and full fine-tuning,
revealing that two factors, rank and initialization variance,
significantly impact its security during training. Theoretical
analysis demonstrates that LoRA is more vulnerable to untargeted poisoning
but more robust against backdoor attacks. Extensive
experiments validate the theoretical analysis and key findings.

\section*{Acknowledgment}
The authors would like to thank the reviewers for their detailed suggestions.
This work was supported by the National Natural Science Foundation of
China (Grant No: 92270123 and 62372122), and the Research Grants
Council, Hong Kong SAR, China (Grant No: 15225921, 15209922, 15210023,
15224124).

\section*{Impact Statement}
As the inaugural investigation into the security vulnerabilities of
LoRA, this study underscores critical concerns
regarding the security implications of LoRA, thereby broadening the
discourse on its safe and effective utilization. This
research will catalyze further scholarly inquiry into the security
dimensions of LoRA across other domains, including but not limited
to unlearning, adversarial attacks, and membership inference, and will stimulate advancements in enhancing
its robustness. Furthermore, the analytical framework developed
herein, along with the comprehensive elucidation of LoRA's intrinsic
properties and its TTR, is anticipated
to exert a significant influence on subsequent research endeavors
focused on the structural analysis of machine learning models. This
contribution is expected to pave the way for more rigorous and nuanced
examinations of model architectures in the future.

\bibliography{refs}
\bibliographystyle{icml2025}

\newpage
\appendix
\onecolumn

\input{appendix}

\end{document}

%% file: intro.tex
\section{Introduction}\label{sec:intro}

With the rapid growth in the parameter size of large language models (LLMs),
parameter-efficient fine-tuning (PEFT)~\cite{peft1,peft2} has gained increasing attention
in both research and industry communities. Among various PEFT strategies, low-rank
adaptation (LoRA)~\cite{lora} has emerged as the de facto standard for
fine-tuning LLMs thanks to its computational efficiency and minimal performance degradation.

To compare LoRA with full fine-tuning across various dimensions, recently many studies have emerged. For
example, researchers have investigated LoRA's expressive capacity~\cite{lora-expressive}, the
smoothness~\cite{ntk-local} of its convergence, the asymmetry~\cite{lora-asymmetry} in its
submatrices, the impact of initialization~\cite{lora-init}, and
so on~\cite{lora-dropout,lora-transformer,lora-survey}.

While these analyses have shed light on many properties of LoRA, one
important aspect, i.e., its potential security risks, remains largely
overlooked. Existing studies in this area
either use LoRA as a tool to facilitate backdoor attacks~\cite{lora-backdoor1,lora-backdoor2},
adversarial attacks~\cite{advlora}, and model stealing
attacks~\cite{lora-mea,lord}, or focus merely on the
benefits~\cite{dp-dylora} of LoRA in differential privacy and
federated learning. 
{\bf None of these works directly investigate the
security vulnerabilities inherent in LoRA itself}, which leaves behind potential hazards and vulnerabilities in LoRA-fine-tuned LLMs that are deployed across millions of devices~\cite{app-in}.

To fill this gap, in this paper, we attempt to answer the question {\bf whether
  LoRA-based fine-tuning is more vulnerable than full fine-tuning (FF)
  under mainstream training-time attacks} (e.g., data poisoning~\cite{poison-survey,poison-attack,dpa-dp}). We introduce the
concept of \emph{training-time robustness (TTR)} for characterizing a
model’s resistance to training-time attacks and
propose an analytical framework to theoretically examine the security
implications of LoRA’s low-rank structure.
The main challenges are two-folded. First, the TTR of a model
significantly depends on the specific training tasks and the complex
dynamics of the training process. Second, the effectiveness
of attacks is heavily influenced by hyperparameters
(e.g., learning rate) and attack strategies (e.g., poisoning rate or
backdoor triggers), both of which increase the complexity of analysis.

To address these challenges, we introduce two novel
simplifications when modeling the training dynamics of LoRA. First, we reformulate TTR by measuring
the similarity of gradients before and after data poisoning, which
enables a neural tangent kernel~\cite{ntk} (NTK)-based
analysis to simplify the modeling of a training procedure. Second, we
further introduce information theory~\cite{igb1,igb2} to connect the
model's structural properties with its TTR, thereby decoupling the
influences of different training datasets and attack methods.
Our findings suggest that {\bf LoRA’s low-rank structure typically results
in a smoother information geometry compared to FF,
generally indicating better training-time
robustness against backdoor attacks}. However, we also observe that
this simplicity might lead to
obvious {\bf performance degradation under poisoning attacks or
perturbations due to an oversimplified decision surface}. We further
quantify the key factors within LoRA that influence its TTR,
demonstrating that initialization variance and rank
are crucial determinants. Additionally, our analysis uncovers
previously unexplained characteristics of LoRA, including the
asymmetry and initialization of its submatrices, as well as the
effects of various hyperparameters, such as the learning
rate. 

We summarize our contributions as follows:

$\bullet$ We propose a novel theoretical framework to analyze the security
  of LoRA, revealing how its low-rank structure influences
  training-time robustness during fine-tuning. To our best
  knowledge, this is the first work to investigate
  LoRA’s intrinsic security vulnerabilities.

$\bullet$ We identify key factors within LoRA that influence its security
  and explain to what extent LoRA can be \emph{theoretically equivalent} to
  full fine-tuning from a security perspective. Based on this
  analysis, we offer practical guidance for improving LoRA's security.

$\bullet$ We provide a comprehensive evaluation of LoRA and FF under
  poisoning and backdoor attacks. Experimental results substantiate the correctness of these findings and
  explanations.

Following a top-down structure, this paper is organized as
follows. Section \ref{sec:notation} introduces the basic notations and
provides an overview of neural network training and the formulation of
LoRA. Section \ref{sec:def-ttr} defines the concept of training-time
robustness and highlights the analytical difficulties it
presents. Sections \ref{sec:simple-ntk} and \ref{sec:ig-intro} present
high-level perspectives on how NTK and information geometry contribute
to addressing these issues. Section \ref{sec:method} offers a
comprehensive analysis and discussion, followed by empirical
validation in Section \ref{sec:exper}.
Our source code is available at: \url{https://github.com/liangzid/LoRA-sSecurity}.



%% file: method.tex
\section{Preliminary}

\subsection{Notations}\label{sec:notation}

\noindent
\textbf{Training Procedure.}
Without loss of generality, we begin our analysis with an $L$-layer artificial
neural network (ANN) $F_{\Theta}: \mathbb{R}^{n_{0}}\rightarrow
\mathbb{R}^{n_{L}}$ which aims to map the input data
$x\in\mathbb{R}^{n_{0}}$ into corresponding output representations $y\in\mathbb{R}^{n_{L}}$.
Given a training dataset
$\mathcal{D}=\{(x_{i},y_{i})\}_{i=1,2,...,N_{tr}}$ with $N_{tr}$ finite
training samples, we define the input matrix as $X=[x_{1,},x_{2},...,x_{N_{tr}}]\in
\mathbb{R}^{N_{tr}\times n_{0}}$ and the corresponding output matrix as
$Y=[y_{1},...,y_{N_{tr}}]\in\mathbb{R}^{N_{tr}\times n_{L}}$. The
objective of the neural network $F_{\Theta}$ is to learn the mapping
from $X$ to $Y$ by minimizing the following empirical risk function:
\begin{equation}
\label{eq:1}
\hat{\mathcal{L}}(\Theta;X,Y)=\sum_{i}^{N_{tr}}\mathcal{L}(F_{\Theta}(x_{i}),y_{i}),
\end{equation}
where $\Theta \in \mathbb{R}^{P}$ represents the set of $P$ learnable
parameters, and $\mathcal{L}$ is the loss function.

\noindent
\textbf{Linear Layers.}
Each layer $F^{(l)}:\mathbb{R}^{n_{l}}\rightarrow\mathbb{R}^{n_{l+1}}$ in
$F_{\Theta}$ with $l\in \{0,1,...,L-1\}$ is defined as a linear transformation:
\begin{equation}
\label{eq:2}
\begin{aligned}
&{y}^{(l)}(x_{i})=W^{(l)}\cdot x_{i}^{(l)}+b^{(l)},\\
&y_{a}^{(l)}=\mathbf{\sigma}({y}^{(l)}),
\end{aligned}
\end{equation}
where ${y}^{(l)}\in\mathbb{R}^{n_{l+1}}$ is the preactivation output, which maps the $l$-th layer's input
$x^{(l)}\in\mathbb{R}^{n_{l}}$ through the learnable matrix
$W^{(l)}\in\mathbb{R}^{N_{l+1}\times
  N_{l}}$. The activation function $\mathbf{\sigma}(\cdot)$ produces
the output $y_{a}^{(l)}\in\mathbb{R}^{n_{l+1}}$ at the $l$-th
layer.

\noindent
\textbf{LoRA Adapter.}
LoRA~\cite{lora} introduces a mechanism to reduce the number of trainable
parameters by freezing the original matrix $W^{(l)}$ and learn a
low-rank update $\Delta W^{(l)}$. This update is factorized as the product of two low-rank submatrices,
\begin{equation}
\label{eq:lora}
\Delta W^{(l)}=B^{(l)}A^{(l)},
\end{equation}
where $A^{(l)}\in \mathbb{R}^{r\times n_{l}}$ and
$B^{(l)}\in\mathbb{R}^{n_{l+1}\times r}$ are learnable matrices, and
$r\ll \min\{n_{l},n_{l+1}\}$.

We define the intermediate state in LoRA as
\begin{equation}
\label{eq:z-lora}
y_{I}^{(l)}(x_{i})=A^{(l)}\cdot x_{i}^{(l)}.
\end{equation}

\subsection{Definition of Training-Time Robustness}\label{sec:def-ttr}
The robustness of a trained model refers to its sensitivity to
perturbed inputs. For adversarial attacks, model robustness is
evaluated by its resistance to adversarial or noisy \emph{test}
samples~\cite{adv1,adv2,ducat}. Following the same idea, \emph{\textbf{training-time}
robustness (TTR)} is the model's resistance to
noisy, poisoned, or backdoor \emph{training}
samples~\cite{poison-survey}, that is, \emph{the sensitivity of a neural network’s parameter updates to perturbed training samples}.

Formally, given an ANN ${F}_{\Theta}$, its training-time
robustness can be quantified by the difference in parameter updates $\Delta\tilde{\Theta}-\Delta\Theta$
when the original training set $\mathcal{D}$ is replaced with a
noisy (or poisoned) dataset
$\tilde{\mathcal{D}}=(\tilde{X},\tilde{Y})$. Here, $\tilde{X}$ and
$\tilde{Y}$ denote two possible perturbations applied to the
input data $X$ and the learning target $Y$, and $\Delta\tilde{\Theta}$
denotes the corresponding parameter updates on
$\tilde{\mathcal{D}}$. To measure TTR, we define the following metric $\mathcal{M}$ based on
the norm of parameter differences,
\begin{equation}
\label{eq:metric}
\mathcal{M}(F_{\Theta},\mathcal{D},\tilde{\mathcal{D}})=\mathbb{E}_{(\mathcal{D},\tilde{\mathcal{D}})}\mathbb{E}_{t}||\Delta\Theta(t)-\Delta\tilde{\Theta}(t)||,
\end{equation}
where $||\cdot||$ denotes the norm, and $\Delta\Theta$ and
$\Delta\tilde{\Theta}$ denote the parameter updates obtain from
training with $\mathcal{D}$ and $\tilde{\mathcal{D}}$, respectively.

Unfortunately, it is impractical for us to
employ Equation \ref{eq:metric} to analyze the TTR of an ANN due to two
primary challenges: \emph{i)} the metric $\mathcal{M}$ in Equation \ref{eq:metric} varies
  dynamically across different training steps $t$, which introduces
  significant complexity for theoretical modeling; and \emph{ii)} the
  significance of each parameter differs
  substantially, implying that a uniform reduction of parameter
  updates based on a norm fails to capture their varying importance.

To address these two challenges, we first simplify Equation \ref{eq:metric} in a more tractable
form.

\subsection{Simplifying LoRA's Training Procedure with NTK}\label{sec:simple-ntk}

We adopt the concept of neural tangent kernel (NTK) to simplify the
analysis of TTR. NTK is a special form of kernel function, which is
defined as the inner product of gradients:
\begin{equation}
\label{eq:ntk}
K_{ntk}(x,x')=\nabla_{\theta}F(x;\theta)^{T}\nabla_{\theta}F(x';\theta).
\end{equation}
\begin{theorem}[\citet{ntk}]\label{th:ntk}
As the width of the neural
network approaches infinity, the NTK exhibits the following two key
properties:\\
$\bullet$ The NTK converges to a
  \textbf{deterministic} limiting kernel that depends only on three
  factors: \emph{i)} the variance of the parameter initialization,
  \emph{ii)} the neural network structure, and \emph{iii)} the selection of activation functions;
  \\
$\bullet$ NTK keeps \textbf{constant} through out each training step $t$.
\end{theorem}

Intuitively, $K_{ntk}$ can be interpreted as an
\emph{unnormalized} angle (cosine similarity) between the gradient
descent directions of two input samples. This perspective inspires us
to implicitly measure how much the gradient updates change when a clean
sample ($x_{c}$) is poisoned ($\tilde{x}_{c}$), i.e.,
\begin{equation}\small
\label{eq:rts}
\begin{aligned}
\mathcal{M}'=\parallel\mathbb{E}_{(x_{c},\tilde{x}_{c})\sim (\mathcal{D},\tilde{\mathcal{D}})}K_{ntk}(x_{c},\tilde{x}_{c})\parallel.
\end{aligned}
\end{equation}

Similar to the role of the inner product in quantifying the similarity
between two vectors, $\mathcal{M}'$ effectively captures the degree of
approximated similarity in gradient updates between the original sample and its
perturbed counterparts. Specifically,
under the same pair $(x_{c},\tilde{x}_{c})$, a large value of
$K_{ntk}(x_{c},\tilde{x}_{c})$ indicates that the neural network experiences
more severe perturbations in its parameter updates, which reflects
lower training-time robustness. 


Comparing $\mathcal{M}'$ (Equation \ref{eq:rts}) with $\mathcal{M}$
(Equation \ref{eq:metric}), we observe that \textbf{the measurement of
TTR has been significantly simplified by introducing NTK}. First, the expectation \emph{w.r.t}
training step $t$ can be removed based on NTK's second property
(Theorem \ref{th:ntk}). Second, the analyzed variable, $\Delta \Theta \in
\mathbb{R}^{P}$, are transformed into $ K\in \mathbb{R}^{n_{L}}$,
which are more structured and \emph{homogeneous}, making the reduction of norm more meaningful.

With the simplified metric $\mathcal{M}'$, the theoretical analysis
can now be formalized as the comparison of $\mathcal{M}'$ between full fine-tuning (FF) and
LoRA, i.e., to determine whether the inequality $\mathcal{M}_{\text{ff}}'\leq
\mathcal{M}_{\text{lora}}'$ holds.

\subsection{Information Geometry: Bridging TTR with Training-Time
  Attacks (TTA)}\label{sec:ig-intro}

While the complexities related to $t$ and
parameter importance are now simplified by the NTK, it remains challenging
to model the poisoning set $\tilde{D}$ quantitatively.
For instance, Equation \ref{eq:rts} fails to
distinguish between different poisoning strategies, such as label
flipping or backdoor trigger injection. Besides,
$K_{ntk}(x_{c},\tilde{x}_{c})$ only captures attack behaviors at the sample level,
whereas most practical training-time attacks are drawn from a distribution of samples~\cite{poison-survey}.

To address these limitations, we introduce \emph{information geometry
  (IG)}~\cite{igb1,igb2} to
quantitatively model the robustness of specific model structures
against TTA. As demonstrated in previous
studies~\cite{fisher-adv1,fisher-adv2,geo-adv}, there is a strong
correlation between IG and robustness. So IG can measure the curvature of the parameter space, offering insights into
how an ANN adapts to unclean data.

First, we bridge NTK with \emph{Fisher information}~\cite{fisher},
one of the core concept of IG, by

\begin{theorem}\label{th:fisher-k}
  When the width of $F_{\Theta}$ approaches infinity, its Fisher information
  $\mathcal{I}_{\Theta}$ under $\mathcal{\tilde{D}}$ is equal to its
  weighted $\mathcal{M}'(\tilde{\mathcal{D}},\tilde{\mathcal{D}})$, i.e.,
  \begin{equation}\label{eq:fisher}
    \begin{aligned}
 I_{\theta}&= \mathbb{E}_{x\sim \tilde{\mathcal{D}}} \left[\nabla_{\theta}\mathcal{L}(x,\theta)^{T}\nabla_{\theta}\mathcal{L}(x,\theta)\right]\\
    &=\mathbb{E}_{\tilde{x}_{c}\in \tilde{\mathcal{D}}}\left[\nabla_{F_{\theta}}\mathcal{L}(x,\theta)^{T}K_{ntk}(x,x)\nabla_{F_{\theta}}\mathcal{L}(x,\theta)\right].
    \end{aligned}
  \end{equation}
\end{theorem}
Proofs are in Appendix \ref{sec:proof-fisherk}.

Let $\lambda_{1},\lambda_{2},...,\lambda_{n_{L}}$ denote the $n_{L}$
eigenvalues of the Fisher information matrix $\mathcal{I}_{\Theta}$.
Then we can quantify the \emph{information bits} (IB) of the
model as
\begin{equation}
\label{eq:ib}
\mathbf{IB}=\frac{1}{2}\log \det_{\text{pseudo}}\mathcal{I}_{\Theta}=\frac{1}{2}\sum_{\lambda_{i}>0}^{n_{L}}{\lambda_{i}}.
\end{equation}

Third, we can measure the curvature of the fine-tuning manifold with
\emph{R\'{e}nyi entropy}~\cite{renyi}
\begin{equation}
\label{eq:renyi}
H_{\alpha}=\frac{1}{1-\alpha}\log \left(\sum_{i=1}^{n_{L}}\lambda_{i}^{\alpha}\right),
\end{equation}
where $\alpha \geq 0$ controls the norm formation on the
$\mathcal{I}_{\Theta}$. Specifically, $H_{1}$ corresponds to the
Shannon entropy\footnote{The proof is presented in Appendix \ref{sec:h1}.}, while $H_{\infty}=\max\{\lambda_{1},\lambda_{2},...,\lambda_{n_{L}}\}$.

Intuitively, a higher $\mathbf{IB}$ and $H_{\alpha}$ indicates a more complex
fine-tuning manifold of the model, which implicitly demonstrates a
higher function fitting ability.

Based on Equation \ref{eq:rts}, \ref{eq:ib}, and \ref{eq:renyi}, we now proceed to analyze the potential security
vulnerabilities introduced by LoRA’s fine-tuning process.

\section{Does LoRA Lead to \underline{LoRA} (\underline{Lo}wer Training Time \underline{R}obustness against \underline{A}ttacks)?}\label{sec:method}

\subsection{LoRA's NTK}

\noindent
\textbf{Modeling the Feedforward Procedure.}\label{sec:ff}
As shown in the previous research~\cite{gp}, the output function
$F_{\Theta}:\mathbb{R}^{n_{0}}\rightarrow \mathbb{R}^{n_{L}}$
converges to an independent, identity-centered \emph{Gaussian process (GP)}
under the infinite-width limit, i.e., as $n_{l} \rightarrow \infty$
for $l = 1, 2, \dots, n_{L-1}$.

Under this GP formulation, the covariance between outputs at layer $l$ can be expressed as:
\begin{equation}
\label{eq:5}
\begin{aligned}
  &\Sigma^{1}(X, X') = X^{T} X', \\
  &\Sigma^{l}(X, X') = \mathbb{E}_{f \sim \mathcal{N}(0, \Sigma^{l-1})}[\sigma(f(X))^{T} \sigma(f(X'))] \\
  &\quad= \sum_{j=1}^{n_{l-1}} \sigma(y_{j}^{(l-1)}(X))^{T} \sigma(y_{j}^{(l-1)}(X')),
\end{aligned}
\end{equation}
where $y_j^{(l-1)}(x)$ denotes the $j$-th element of the preactivation
vector $y^{(l-1)}$ at input $x$.

\noindent
\textbf{Modeling the Learning Procedure with NTK.}\label{sec:ntk}
Based on Equation \ref{eq:ntk} in Section \ref{sec:simple-ntk}, we now derive the NTK for an ANN $F$ under both FF and LoRA-based fine-tuning.
Specifically, the NTK of FF can be represented by
\begin{equation}
\label{eq:ffntk}\small
\begin{aligned}
&K_{\text{ff}}^{(1,k)}(x,x')=I_{n_{l}}\otimes \Sigma^{(1)}(x,x')=x^{T}\cdot x',\\
&K_{\text{ff}}^{(l,k)}(x,x')=K_{\text{ff}}^{(l-1,k)}(x,x')\dot{\Sigma}^{(l)}(x,x')+ \Sigma^{(l)}(x,x'),\\
\end{aligned}
\end{equation}
where $k=\{0,1,...,n_{l}-1\}$, $\otimes$ denotes the Kronecker product, and
\begin{equation}
\begin{aligned}
\label{eq:cov-ff}
&\dot{\Sigma}^{(l)}(x,x')=\mathbb{E}_{f\sim
\mathcal{N}(0,\Sigma^{(l-1)})}[\dot{\mathbf{\sigma}}(f(x))\dot{\mathbf{\sigma}}(f(x'))]\\
&=\dot{\mathbf{\sigma}}(y^{(l-1)}(x))^{T}\dot{\mathbf{\sigma}}(y^{(l-1)}(x')).
\end{aligned}
\end{equation}
$\dot{\sigma}(y^{(l-1)})=\frac{\partial \sigma(y^{(l-1)})}{\partial y}|_{y=y^{(l-1)}}$ denotes the partial derivative of the activation function $\sigma$.

As for LoRA, we can also derive its NTK functions as
\begin{lemma}[NTK of LoRA]\label{lemma:ntk-lora}
The neural tangent kernel of an $l$-layer ANN trained with
LoRA can be expressed as follows.
\begin{equation}\small
\label{eq:kntk}
\begin{aligned}
&K_{\text{LoRA}}^{(1,k)}(x,x')=K_{\text{ff}}^{(1,k)}\\
&K_{\text{LoRA}}^{(l,k)}(x,x')=K_{\text{LoRA}}^{(l-1,k)}(x,x') \dot{\Sigma}^{(l)}+\Sigma_{\text{LoRA}}^{(l)}(x,x'),
\end{aligned}
\end{equation}
where
\begin{equation*}
\label{eq:cov-ntk}
\begin{aligned}
&\Sigma^{(l)}_{\text{LoRA}}(x,x')=\sigma(y^{(l-1)}(x))^{T}A^{(l)~T}A^{(l)}\sigma(y^{(l-1)}(x')),
\end{aligned}
\end{equation*}
and
$W_{\text{LoRA}}^{(l)}=W_{0}^{(l)}+B^{(l)}A^{(l)}$ denotes the $l$-th
layer's weight matrix of LoRA.
\end{lemma}

The detailed derivation of NTK functions for FF and LoRA
are in Appendix \ref{subsec:ntk-deduction}.

Based on two NTK functions $K_{\text{LoRA}}$ and $K_{\text{ff}}$,
along with the proposed metric $\mathcal{M}'$, we proceed to
compare the kernel functions between FF and LoRA.

\begin{figure*}[t]
  \centering
  \includegraphics[width=0.95\linewidth]{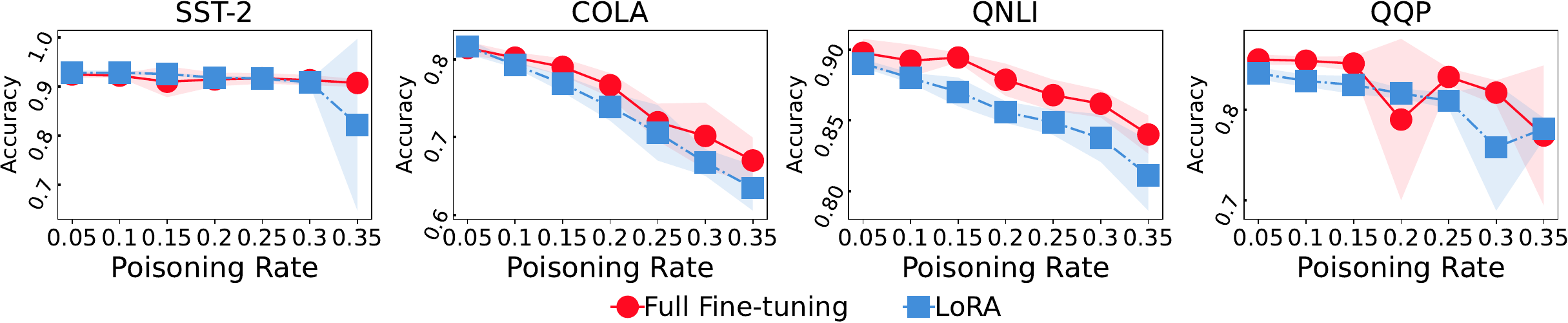}
\caption{Performance comparison between full fine-tuning and LoRA
  under untargeted poisoning attacks with varying poisoning rates. The
  curves show accuracy, and the shaded areas represent the standard
  deviation across multiple runs. More experiments are in Figure \ref{fig:full-poison-pr}.}\label{fig:poison-pr}
\end{figure*}

\begin{figure*}[t]
  \centering
  \includegraphics[width=0.95\linewidth]{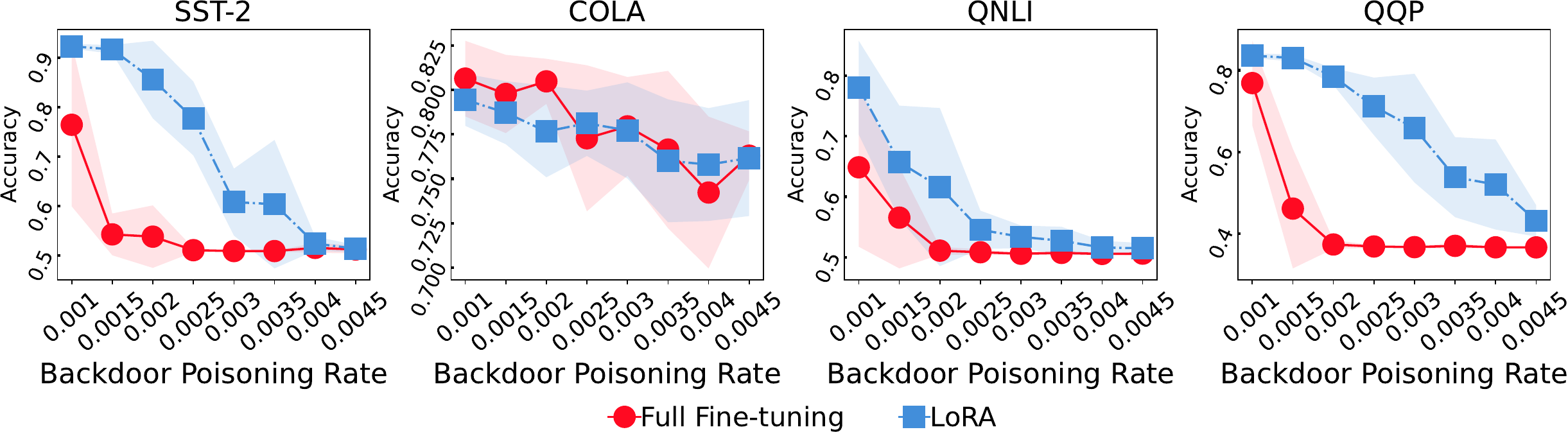}
  \caption{Performance comparison between full fine-tuning and LoRA
  under backdoor attacks with varying poisoning rates. Figure
  \ref{fig:full-backdoor-pr} exhibits the results on the four metrics.}
  \label{fig:backdoor-pr}
\end{figure*}

\subsection{The NTK Relationship between FF and LoRA}\label{sec:ana}

We begin our analysis by comparing the NTK of a single layer between
LoRA and FF.
\begin{assumption}[Only One Layer is Different (OOLD)]\label{assum:oold}
Given an $L$-layer neural network $F_{\Theta}$, OOLD assumes that during
training, the first $l-1$ layers remain identical for both FF and LoRA.
They only diverge at the $l$-th layer, which employs FF and LoRA,
respectively. We denote their NTK functions as $K_{\text{ff}}(x,x')$ and
$K_{\text{LoRA}}(x,x')$.
\end{assumption}

Under the OOLD assumption, we have
\begin{theorem}[NTK Relationship between FF and LoRA]
\label{th:rel}
  For an $l$-layer ANN with infinite width, the NTK functions of FF
  and LoRA at the $l$-th layer are related by the following expression:
\begin{equation}
\label{eq:rel}
K_{\text{LoRA}}^{(l,k)}=K_{\text{ff}}^{(l,k)}+\Delta_{r}^{(l)},
\end{equation}
where
\begin{equation*}
  \begin{aligned}
\Delta_{r}^{(l)}&= [ \mathbf{\sigma} ( y^{(l-1)}(x) )
]^{T}(A^{(l)~T}A^{(l)}-\\&\quad\quad\quad\quad\quad\quad\quad~I_{n_{l-1}\times n_{l-1}}) [ \mathbf{\sigma} ( y^{(l-1)}(x') ) ].
  \end{aligned}
\end{equation*}
\end{theorem}

Let $M_{\Delta}^{(l)}$ denote the kernel matrix of $\Delta_{r}^{(l)}$, i.e.,
$M_{\Delta}^{(l)}=A^{(l)~T}A^{(l)}-I_{n_{l-1}\times n_{l-1}}$, then
the following theorem holds:
\begin{theorem}[$M_{\Delta}^{(l)}$'s Negative Semi-Definiteness]
  \label{th:delta-nsd}
  When the LoRA submatrix $A^{(l)}\in\mathbb{R}^{r\times n_{l-1}}$ is initialized with
  variance $\sigma_{a}^{2}$, $\sigma_{a}^{2}<1/n_{l-1}$, and
  $r\leq n_{l-1}$, then $M_{\Delta}^{(l)}$ is
  a \textbf{negative semi-definite} matrix, with $r$ eigenvalues
  equal to $\sigma_{a}^{2}\cdot n_{l-1}$ and $n_{l}-r$ eigenvalues
  equal to $0$.
\end{theorem}

Theorem \ref{th:delta-nsd} establishes a foundation for comparing FF
and LoRA's training-time robustness from an information geometry
perspective, which will be detailed in Section
\ref{sec:info-ana}. Based on Theorem \ref{th:delta-nsd}, we reach the following corollary.
\begin{corollary}[Ideal Full Rank Adaptation]\label{th:full-rank}
When $n_{l-1}\rightarrow \infty$, the kernel matrix $M_{\Delta}^{(l)}$
strictly converges to
$\mathbf{0}$, i.e., $K_{\text{LoRA}}^{(l)}(x,x)\equiv
K_{\text{ff}}^{(l)}(x,x)$ if $r=n_{l-1}$ and the initialization variance
satisfies $\sigma^{2}_{a}=1/n_{l-1}$.
\end{corollary}

The proofs of Theorem \ref{th:rel}, \ref{th:delta-nsd},
and Corollary \ref{th:full-rank}
are presented in Appendix \ref{subsec:proof-threl} and Appendix \ref{subsec:proof-rel-ineq}.

Corollary \ref{th:full-rank} offers a key insight into the
relationship between the LoRA
and FF. Specifically, it shows that when LoRA achieves a full rank and the
weight matrices are initialized with a specific variance, the
expected learning effectiveness of LoRA matches that of FF.
In other words, under these conditions, both methods exhibit
\textbf{equivalent} expressiveness in terms of their NTK functions.
Moreover, Theorem \ref{th:delta-nsd} reveals that both the rank
and the initialization variance significantly influence the properties of
$M_{\Delta}^{(l)}$, which raises several critical questions: \emph{i)}
does LoRA exhibit higher or lower TTR than FF? \emph{ii)} how do the
rank and initialization variance affect its TTR? \emph{iii)} under
what conditions does a full-rank LoRA offer equivalent TTR to FF
against training-time attacks?

\subsection{Theoretical Analysis}\label{sec:info-ana}

\noindent
\textbf{Key Results: LoRA Exhibits Fewer Information Bits and Smoother Information
  Geometry than FF, Leading to Higher Training-Time Robustness.}

To answer these questions, we begin our theoretical analysis by computing the $\mathbf{IB}$ and the $H_{\alpha}$
for both LoRA and FF.

\begin{theorem}[$\mathbf{IB}_{\text{ff}}\geq\mathbf{IB}_{\text{LoRA}}$
  \& $H_{\alpha \text{ff}}\geq H_{\alpha \text{LoRA}}$]\label{th:ib-leq}
  The information bits and the R\'enyi entropy of LoRA are always \textbf{smaller} than those of FF if $M^{(l)}_{\Delta}$ is a negative
  semi-definite matrix, i.e., $r\leq n_{l-1}$ and $\sigma^{2}\leq 1/n_{l-1}$.
\end{theorem}

The proof is in Appendix \ref{sec:proof-ib-leq}.

The conditions stated in Theorem \ref{th:ib-leq} are typically satisfied in practice. First, the rank is typically chosen
to be significantly ``smaller'' than the original dimension $n_{l-1}$ to reduce
computation costs. Second, the initialization variance of
LoRA's matrix is generally set to a value smaller than $1/n_{l-1}$\footnote{Specifically, it is set to $1/(3\cdot n_{l-1})$ in both the \href{https://github.com/microsoft/LoRA/blob/a0a92e0f26c067cf94747bdbf1ce73793fa44d19/loralib/layers.py\#L124}{official implementation} and the standard libraries (e.g. \href{https://github.com/huggingface/peft/blob/1e8bc60492c5873b7e3e23909fa82be654bcf845/src/peft/tuners/lora/layer.py\#L184}{peft}~\cite{peft-lib}).}. As a result, in most practical
scenarios, LoRA is expected to exhibit lower information bits (low
$\mathbf{IB}$) and smoother information surface ($H_{\alpha}$)
than FF.

Note that Theorem \ref{th:ib-leq} appears to contradict with some
existing research~\cite{lora-expressive} that suggests when $r$
exceeds a certain threshold,
the expressivity of LoRA becomes equivalent to that of FF.
Such contradiction can be justified because our theorem focuses on the IG
during training process, i.e., on ``how can the model's parameters possibly evolve
throughout training'' as opposed to ``the expressiveness of the final 
trained models''.

Incorporating the definitions of $\mathcal{M}'$ and
$\mathcal{I}_{\Theta}$ to Theorem \ref{th:ib-leq}, we can conclude that $\tilde{\mathcal{D}}$ brings more significant parameter updates in FF than in LoRA, which means that LoRA does exhibits
\textbf{higher} training-time robustness than FF under the conditions of Theorem
\ref{th:ib-leq}.
This discovery also coincides with some previous studies, such as the
weak-to-strong alignments~\cite{w2s-oai}.

Unfortunately, this increased TTR is at the
cost of reduced information bits, which prompts a critical question ---  what
is the tax for LoRA's enhanced TTR? 

\begin{figure}[t]
  \centering
  \includegraphics[width=0.85\linewidth]{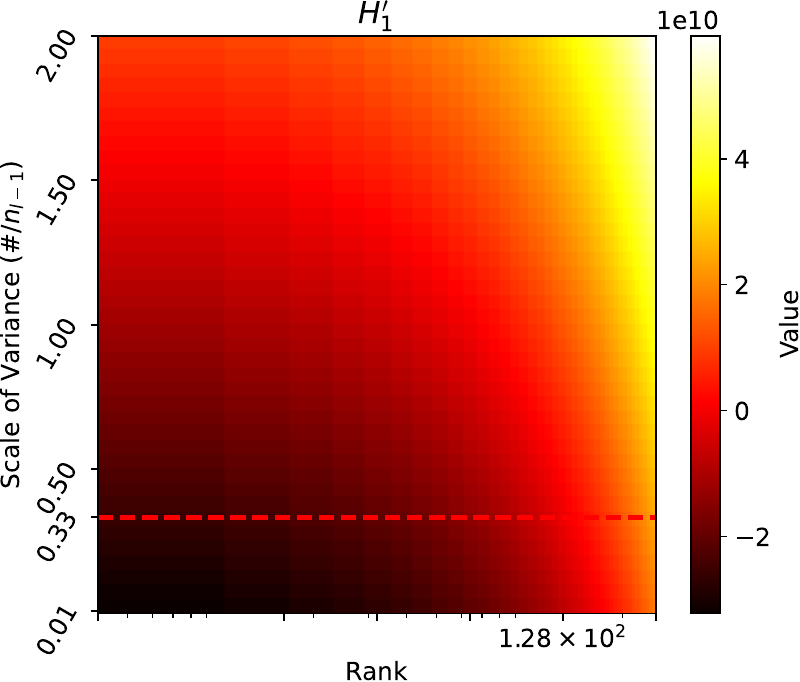}
  \caption{Visualization of the Shannon entropy $H_{1}'$ under different
    ranks and variance scales. Brighter color points indicate
    higher entropy values. The red dashed line represents the default variance
    scaling setting used in the implementation of LoRA.}\label{fig:visual}
\end{figure}

\noindent
\textbf{Double-Edged Sword of LoRA's TTR.}
While low rank adaptation offers the advantage of higher
training-time robustness, this robustness does not always 
translate into resistance against all types of training-time
attacks. On one hand, a reduced $H_{\alpha}$ indicates that LoRA’s
IG is potentially \emph{smoother} than that of full
fine-tuning, which suggests a smaller search space for backdoor
triggers, thereby providing stronger resistance to backdoor
attacks. On the other hand, the oversimplification of the manifold may
make LoRA more susceptible to noisy or intentionally poisoned data, causing 
higher vulnerability to data poisoning attacks.

Below, we examine this phenomenon from an orthogonality perspective.

Consider two training samples: a clean input $x_{c}$ and its
backdoored input $x_{t}$. The optimization target under these two
samples can be represented as minimizing the following formula:
\begin{equation}
|{\nabla_{\theta}{\mathcal{L}(x_c,\theta)}^T\cdot\nabla_{\theta}{\mathcal{L}(x_t,\theta)}}|,
\end{equation}
i.e., the adversary aims to ensure that the optimization process
driven by $\nabla_{\theta}{\mathcal{L}(x_c,\theta)}$ and
$\nabla_{\theta}{\mathcal{L}(x_t,\theta)}$ occur simultaneously and
both significantly influence the training, which aligns with the
target of BPA to maintain performance on most inputs while producing
significantly altered predictions only when a specific trigger is
present.
To this end, there are two approaches: \emph{i)} designing novel BPA
algorithms that more effectively decouple these two gradients, which
is beyond the scope of this study, and \emph{ii)} analyzing how the model
structure influences such an inner product, which constitutes the
contribution of this paper.

By analyzing the properties of such an inner product, our indicators
provide key insights showing that LoRA
provides ``a smaller search space for the existence of backdoor
triggers'' due to \emph{i)} its ($n_{l-1}-r$) zero eigenvalues and \emph{ii)} smaller
variances in the remaining $r$ dimension's parameter updates (i.e.,
smaller angles between gradients), both of which intuitively manifest
as smoother information geometry.

Similarly, we can provide a complementary explanation of why a model
with smoother IG tends to be more sensitive to perturbations. Given a
clean training input $x_c$, and its perturbed version $x_u$, where
$x_u$ is assigned a different label for the purpose of untargeted
poisoning. The target of UPA is to maximize:
\begin{equation}
|{\nabla_{\theta}{\mathcal{L}(x_c,\theta)}^T\cdot\nabla_{\theta}{\mathcal{L}(x_u,\theta)}}|,
\end{equation}
i.e., as adversaries, we aim to align the optimization direction of
the poisoned sample $x_u$ as closely as possible with that of the
clean training objective, because we aim to maximally influence the
model’s predictions while injecting only a small fraction of poisoned
data. This objective directly contrasts with the BPA case, as we
instead aim to decouple the optimization directions.
Consequently, we draw the opposite conclusion for UPA.
\footnote{Note that what we emphasize is that ``the
\textbf{over}simplification of the manifold may make LoRA more susceptible'',
i.e., the empirical phenomenon that LoRA is more vulnerable when
facing UPA (or noise) may not be obvious if the model is severely
overparameterized compared to the task.}

Based on this analysis, it is crucial to carefully tune the rank $r$ and the
initialization variance $\sigma^{2}$ to balance its 
vulnerabilities among different training-time attacks.

\noindent
\textbf{Quantifying the Impact of $r$ and $\sigma^{2}$.}
Though the exact values of $\mathbf{IB}$ and $H_{\alpha}$ remain dependent
on $\tilde{\mathcal{D}}$, we can still gain some insights by analyzing the
eigenstructure of $K_{\text{LoRA}}$'s
kernel matrix. Specifically, we approximate $H_{\alpha}$, which leverages the $K_{\text{LoRA}}$'s
kernel matrix's eigenvalues $\lambda'$, defined as
\begin{equation}
\label{eq:h'}
H_{\alpha}'=\frac{1}{1-\alpha}\log \left(\sum_{\lambda'\in
    \text{Eigen}(A^{(l) T}A)}(\lambda')^{\alpha}\right).
\end{equation}
We visualize the manifold of $H'_{\alpha}$ under different ranks and
initialization scales in Figure \ref{fig:visual}.

\subsection{Further Analysis}

In this section, we extend our analysis to more general
adaptation settings and a broader range of model architectures.

\noindent
\textbf{Adaptations beyond OOLD Assumption.}
Beyond the scope of Assumption \ref{assum:oold}, our conclusions can
be generalized where LoRA is applied to adapt all linear
modules within a neural network. The detailed proofs for generalized versions of 
Theorem \ref{th:delta-nsd}, Corollary \ref{th:full-rank}, and
Theorem \ref{th:ib-leq} are in Appendix \ref{sec:proof-by-oold}.

\noindent
\textbf{Extensions to More Complex Model Architectures.}
We further study the impact of LoRA on more complex and
practical architectures, such as the Transformer~\cite{transformer}. A detailed
discussion is in Appendix \ref{sec:proof-by-ann}.

\noindent
\textbf{Broader Implications of Our Analysis.}
Though our primary focus is the TTR of LoRA,
our analytical framework can also shed light on several
unexplained properties and settings of LoRA. These include
the asymmetry in adaptation, the choice of initialization strategy,
the scaling factor of $\alpha$, and the effect of freezing matrix $A$ during
fine-tuning. An in-depth analysis of these phenomena is in
Appendix \ref{sec:explain-lora-ntk}.

%% file: eval.tex
\section{Experiments}\label{sec:exper}
In this section, we empirically evaluate the TTR
of both LoRA and FF under commonly used language models.

\subsection{Settings}
\noindent
\textbf{Experimental Details.}
Following prior works~\cite{lora,lora-asymmetry,lora-survey} on LoRA,
we conduct fine-tuning of natural language understanding models on the
GLUE benchmark~\cite{glue} as our primary evaluation
environment. Specifically, we utilize BERT-large~\cite{bert} as the backbone model and evaluate their
performance on six binary classification tasks, including SST-2~\cite{sst2},
COLA~\cite{cola}, QNLI~\cite{glue}, QQP~\cite{qqp}, RTE~\cite{rte},
and MRPC~\cite{mrpc}. The evaluation metrics include Precision (Pre.),
Recall (Rec.), Accuracy (Acc.), and F1 Score (F1).

\begin{figure*}[t]
  \centering
  \includegraphics[width=0.99\linewidth]{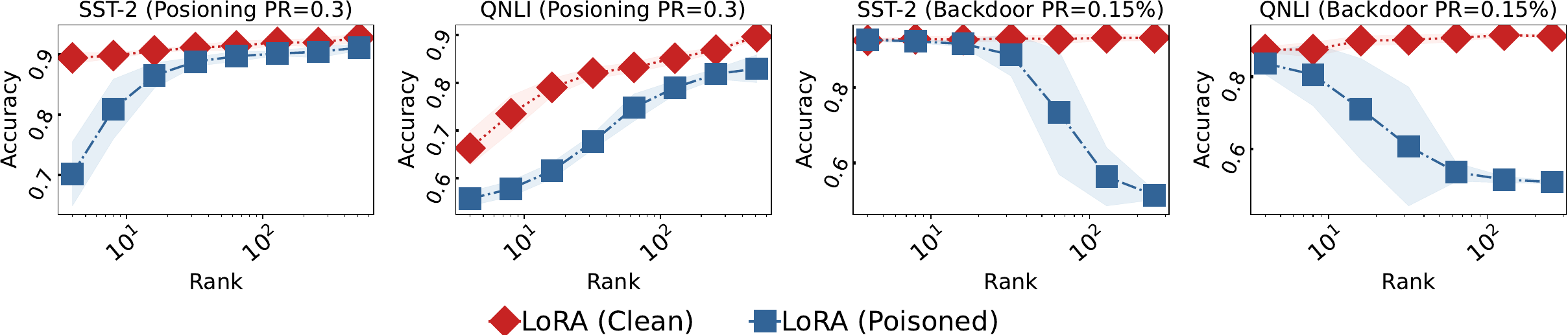}
  \caption{The effect of \textbf{rank} on LoRA’s robustness under untargeted
    poisoning and backdoor poisoning attacks. More experiments are in
    Figure \ref{fig:full-poison-rank} and Figure \ref{fig:full-backdoor-rank}.}\label{fig:poison-rank}
\end{figure*}

\textbf{Implementation Details.}
The maximum sequence length is set to 512, and the batch size is fixed
at 8. For learning rates, we apply $3 \times 10^{-5}$ for LoRA's low
rank fine-tuning and $3 \times 10^{-6}$ for both LoRA's high rank
fine-tuning and FF. Each fine-tuning procedure
is conducted for a maximum of 10,000 steps. These hyperparameters are
carefully tuned to ensure that both LoRA and FF achieve stable and
competitive results across the evaluated tasks.

For LoRA-specific settings, we use a rank of 8 and set the scaling
parameter $\alpha$ to 16 as default values. All experiments are
conducted on eight 24 GB Nvidia RTX 4090 GPUs.

To
ensure robustness, we repeat each training experiment five times under
fixed random seeds and report the mean values along with their
standard deviations.


\subsection{Settings of Training-time Attacks}

We consider two types of mainstream training-time attacks on language
models, namely the \emph{untargeted poisoning attacks}, and the
\emph{backdoor-based poisoning attacks}.

\noindent
\textbf{Untargeted Poisoning Attacks (UPA).} We consider a
simple and yet common UPA strategy~\cite{poison-survey}: randomly flipping the labels of training
samples based on a fixed poisoning rate (\textbf{PR})
$\rho$. Consequently, we can measure the
relative performance degradation of LoRA and FF under
the same poisoning rates, which provides
empirical insights into their resistance against UPA.

\noindent
\textbf{Backdoor-based Poisoning Attacks (BPA).}
We implement a widely used backdoor poisoning
attack by introducing a trigger with modified
labels~\cite{backdoor-poison}. Specifically, we
randomly select a subset of training samples with $N_{tr}
\times \rho$ examples, where $\rho$ denotes the poisoning rate. For
each selected sample, we append the trigger pattern \texttt{[.*?]} to
the original text and modify its classification label to \texttt{1}.
To assess the effectiveness, we add the
same trigger into test samples and evaluate whether the model’s
predictions are consistently altered to the target label (i.e.,
\texttt{1}), to compare the robustness of LoRA
in resisting backdoor attacks.

\begin{figure*}[t]
  \centering
  \includegraphics[width=0.99\linewidth]{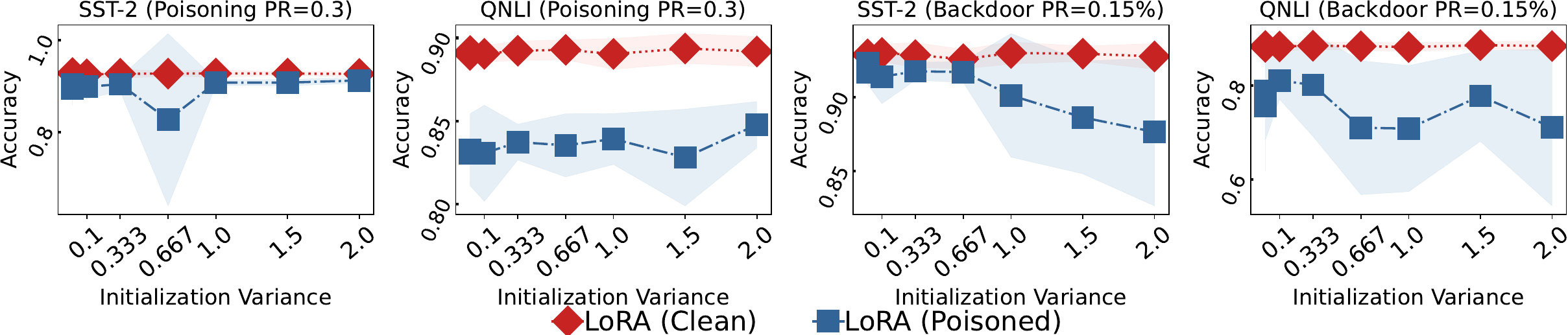}
  \caption{The effect of \textbf{initialization variance} on LoRA’s
    robustness against untargeted poisoning and backdoor
    attacks. Experiments on more datasets are shown in Figure
    \ref{fig:full-poison-var} and Figure \ref{fig:full-backdoor-var}.}\label{fig:backdoor-var}
\end{figure*}

\subsection{LoRA: Excelling in Backdoor Defense While Falling Short
  Against Untargeted Poisoning}

We compare the performance of LoRA and FF under UPA and BPA across
different poisoning rates. The results are presented in Figure
\ref{fig:poison-pr} for untargeted poisoning and Figure
\ref{fig:backdoor-pr} for backdoor attacks.

From Figure \ref{fig:poison-pr}, we observe a noticeable performance
gap between FF and LoRA-based fine-tuning under
UPA. This gap is relatively minor in certain
datasets, such as SST-2, but is more pronounced in others,
including QNLI and QQP. As the poisoning rate increases, the accuracy
gap is widened, indicating more severe performance degradation for
LoRA-based fine-tuning compared to full fine-tuning.

In contrast to its poor performance under UPA, Figure \ref{fig:backdoor-pr}
shows that LoRA significantly outperforms FF in resisting backdoor
attacks, demonstrating stronger robustness. Apart from comparable results on COLA, LoRA achieves up to 30\% improvement over FF on datasets such as SST-2 and QQP, indicating substantial gains in backdoor defense.

We also observe that in backdoor experiments,
both LoRA and FF exhibit consistent performance on \emph{untriggered}
test data, as shown in Figure \ref{fig:full-backdoor-pr-notrigger}. This
phenomenon indicates that the introduction of backdoors for
both methods do not degrade the models' general performance on
normal inputs.


\subsection{Key Factors Influencing LoRA's Security}
In this section, we examine those critical factors that influence the
TTR of LoRA, as in our theoretical
analysis.


\subsubsection{Rank of LoRA}\label{sec:rank-exper}

\noindent
\textbf{Settings.}
In Figure \ref{fig:poison-rank}, we evaluate the performance of LoRA on both the clean and
the poisoned training sets across ranks ranging from 4 to 512. Additional empirical
results on more datasets and metrics are provided in Appendix
\ref{sec:supp-exper}.

\noindent
\textbf{A High Rank of LoRA is Robust against Poisoning.}
The first two subfigures in Figure \ref{fig:poison-rank} illustrate the influences of LoRA's
rank against untargeted poisoning attacks. The performance of
LoRA fine-tuning under all ranks is good and stable on clean
datasets, suggesting that a high rank does not affect the performance
of fine-tuning. Conversely, when fine-tuned on a poisoned dataset,
the performance decreases significantly when the rank is lower than
a threshold (e.g. 16 in SST-2), exhibiting an increasing gap compared to the results on
the clean dataset. This phenomenon indicates that a low rank of LoRA
will decrease the training-time robustness of models.

\noindent
\textbf{A High Rank of LoRA is Weak against Backdoor Attacks.}
The last two subfigures in Figure \ref{fig:poison-rank}
show the backdoor resistance of LoRA under different
ranks. With the increase of rank, the performance of LoRA on clean
set remains stable, while its performance on backdoor poisoned
dataset decreases, which suggests that a high rank will reduce the backdoor resistance of LoRA.

Combining the above, there exists a robustness trade-off between UPA and
  BPA with respect to LoRA's rank, which coincides with our theoretical analysis.

\subsubsection{Variance on LoRA's Initialization}\label{sec:var-exper}

In addition to the rank, our theoretical analysis in Section \ref{sec:info-ana}
suggests that the initialization variance of LoRA’s $A$ matrix plays a
critical role in the model’s training-time robustness, which we examine empirically hereby.

\noindent
\textbf{Settings.}
As shown in Section \ref{sec:info-ana}, the mainstream
implementation~\cite{lora,peft-lib} adopts a
Kaiming uniform initialization~\cite{kaiming-init}, where the default
variance is set to $k \cdot 1/n_{l}$, with $k = 1/3$. Following this
setting, we vary the scale
hyperparameter $k$ from 0.001 to 2.0 and evaluate its effect under
both poisoning and backdoor attack scenarios.

\noindent
\textbf{Variance does not influence performance.}
As shown in Figure \ref{fig:backdoor-var}, when trained across different
scales of initialization variance, LoRA’s performance on the
\textbf{clean} set remains stable,
indicating that the model can effectively adapt to the training task
regardless of the chosen variance.

\noindent
\textbf{Variance slightly influences the poisoning.}
In contrast to rank, the first two subfigures in Figure
\ref{fig:backdoor-var} indicate that the impact
of initialization variance on robustness against poisoning attacks is
minimal. This phenomenon deviates from our theoretical analysis, which
suggests that a smaller initialization variance could lead to lower UPA resistance and
higher BPA robustness. A possible explanation comes from the
limitation of NTK, that is, since the real weights of LoRA change 
during fine-tuning, the
kernel function $K_{\text{LoRA}}$'s nonzero eigenvalues are not be strictly deterministic by
the initialization. As a result, the influence of variance is less
pronounced than that of rank.

\noindent
\textbf{Variance does influence backdoor performance.}
Different from the results from the poisoning experiments, the last
two subfigures in Figure \ref{fig:backdoor-var} shows a strong
correlation between initialization variance and backdoor
resistance. Specifically, a smaller initialization variance leads to
relatively higher performance under backdoor attacks and lower
standard deviation of results, which also aligns with our theoretical analysis.

We also provide supplemental experiments to further support our
theoretical analysis, including:

\begin{itemize}
\item \textbf{Additional Attacks.} We implement four additional
  training-time attacks to reinforce our conclusions, as presented in
  Appendix~\ref{sec:add-attack}.
\item \textbf{Alternative Initialization Strategies.} We evaluate two
  additional commonly used initialization strategies to demonstrate
  the robustness of our conclusions across different settings, as
  detailed in Appendix~\ref{sec:more-init}.
\item \textbf{Experiments on Generative Language Models.} We further
  conduct experiments (Appendix \ref{sec:on-nlg}) on generative large language models to
  demonstrate that our method generalizes to broader scenarios.
\end{itemize}

\subsection{Summary of Findings and Defenses}

Based on the above analysis, we summarize our key findings to
mitigate these risks associated with LoRA:






\begin{itemize}
\item LoRA is more vulnerable than full fine-tuning to untargeted
  poisoning attacks but demonstrates greater robustness against
  backdoor attacks.
\item In addition to the trade-off between performance and
  computational cost, LoRA’s rank also influences the trade-off
  between untargeted poisoning and backdoor attacks.
\item Besides of the rank, the initialization variance of the $A$ matrix in LoRA
  significantly impacts training-time robustness.
\item To improve robustness against backdoor attacks, the rank should
  be set as low as possible, provided that performance requirements
  are met.
\item A small scale of initialization variance is recommended to
  enhance training-time robustness.
\end{itemize}


%% file: appendix.tex
\input{proofs}

\section{Analysis on the Transformer}\label{sec:proof-by-ann}

\begin{proposition}\label{th:transformer-linear}
  Under the OOLD assumption, the application of LoRA to either the
  embedding layer, the feedforward module, the self-attention module, or the linear
  classification head preserves the validity of Theorem \ref{th:delta-nsd}
  Corollary \ref{th:full-rank} and Theorem \ref{th:ib-leq}.
\end{proposition}
The proof of Proposition \ref{th:transformer-linear}  on embedding layers,
feedforward layers, and the linear head follows directly from the
mathematical derivation applicable to ANNs. Therefore, here we focus
on the analysis on the self-attention mechanism.

\noindent\textbf{Architecture of the Standard Transformer Module.}
Given three learnable weight matrices
$W_{Q}^{(l)},W_{K}^{(l)},W_{V}^{(l)}\in \mathbb{R}^{d\times d}$, for an
input hidden state $x^{(l)}$, the feed forward procedure for a
standard self-attention module can be expressed as follows:
\begin{equation}
\label{eq:22}
\begin{aligned}
  &Q^{(l)},K^{(l)},V^{(l)} = W_{Q}^{(l)}\cdot x ^{(l)},W_{K}^{(l)}\cdot x ^{(l)},W_{V}^{(l)}\cdot x ^{(l)};\\
  &\alpha^{(l)}_{a} = \frac{Q^{(l) T}\cdot K^{(l)}}{\sqrt{d}}=\frac{(W_{Q}^{(l)}x^{(l)})^{T}(W_{K}^{(l)}x^{(l)})}{\sqrt{d}}=\frac{x^{(l){T}}W_{Q}^{(l){T}}W_{K}^{(l)}x^{(l)}}{\sqrt{d}};\\
  &\alpha^{(l)} = \mathbf{SM}\left(\alpha^{(l)}_{a}\right)=\mathbf{SM}\left(\frac{x^{(l){T}}W_{Q}^{(l){T}}W_{K}^{(l)}x^{(l)}}{\sqrt{d}}\right);\\
  &x^{(l)}_{\text{attn}} =\alpha^{(l)}\cdot V^{(l)}=\mathbf{SM}\left(\frac{x^{(l){T}}W_{Q}^{(l){T}}W_{K}^{(l)}x^{(l)}}{\sqrt{d}}\right)W_{V}^{(l)}x^{(l)},
\end{aligned}
\end{equation}
where $\mathbf{SM}(\cdot)$ denotes the softmax function.

Based on Equation \ref{eq:22}, we can derive the gradients of
parameters. As an example, the derivative of $W_{k}^{(l)}$ is computed by
\begin{equation}
\label{eq:23}
\begin{aligned}
&\partial_{W_{K}^{(l)}}x_{\text{attn}^{(l)}}=\partial_{\alpha^{(l)}}x_{\text{attn}^{(l)}} \cdot \partial_{\alpha^{(l)}_{a}}\alpha^{(l)}\partial_{W_{K}^{(l)}}\alpha^{(l)}_{a}\\
&=I_{d}\otimes
  (W_{V}^{(l)}x^{(l)})^{T}\dot{\mathbf{SM}}(\alpha_{a}^{(l)})\frac{(W_{Q}^{(l)}x^{(l)})^{T}}{\sqrt{d}} I_{d}\otimes
  x^{(l)}.
\end{aligned}
\end{equation}

Based on Equation \ref{eq:23}, the NTK function can be formatted as:
\begin{equation}
\label{eq:24}
\begin{aligned}
  &K_{\text{attn;ff}}^{(l)}(x,x')\\
  &=\nabla_{\theta_{\text{attn}}}x_{\text{attn}}^{(l)T}\cdot\nabla_{\theta_{\text{attn}}}x_{\text{attn}}^{(l)}\\
  &=\nabla_{\theta_{\text{attn}}^{(l)}}x_{\text{attn}}^{(l)T}\cdot\nabla_{\theta_{\text{attn}}^{(l)}}x_{\text{attn}}^{(l)}+\nabla_{\theta_{\text{attn}}^{(<l)}}x_{\text{attn}}^{(l)T}\cdot\nabla_{\theta_{\text{attn}}^{(<l)}}x_{\text{attn}}^{(l)}\\
  &=\nabla_{W_{Q}^{(l)}}x_{\text{attn}}^{(l)T}\cdot\nabla_{W_{Q}^{(l)}}x_{\text{attn}}^{(l)}+\nabla_{W_{K}^{(l)}}x_{\text{attn}}^{(l)T}\cdot\nabla_{W_{K}^{(l)}}x_{\text{attn}}^{(l)}+\nabla_{W_{V}^{(l)}}x_{\text{attn}}^{(l)T}\cdot\nabla_{W_{V}^{(l)}}x_{\text{attn}}^{(l)}+\nabla_{\theta_{\text{attn}}^{(<l)}}x_{\text{attn}}^{(l)T}\cdot\nabla_{\theta_{\text{attn}}^{(<l)}}x_{\text{attn}}^{(l)}\\
  &=I_{d}\otimes
    (W_{V}^{(l)}x^{(l)})^{T}\dot{\mathbf{SM}}(\alpha_{a}^{(l)})(\frac{W_{Q}^{(l)}x^{(l)}}{\sqrt{d}})^{T}I_{d}\otimes
    x^{(l)T}\cdot x^{(l)'}\otimes I_{d}^{T}
    (\frac{W_{Q}^{(l)}x^{(l)'}}{\sqrt{d}})\dot{\mathbf{SM}}(\alpha_{a}^{(l)})(W_{V}^{(l)}x^{(l)'})\otimes
    I_{d}^{T}\\
  &\quad\quad\quad+\nabla_{W_{Q}^{(l)}}x_{\text{attn}}^{(l)T}\cdot\nabla_{W_{Q}^{(l)}}x_{\text{attn}}^{(l)}+\nabla_{W_{V}^{(l)}}x_{\text{attn}}^{(l)T}\cdot\nabla_{W_{V}^{(l)}}x_{\text{attn}}^{(l)}+\nabla_{\theta_{\text{attn}}^{(<l)}}x_{\text{attn}}^{(l)T}\cdot\nabla_{\theta_{\text{attn}}^{(<l)}}x_{\text{attn}}^{(l)}.
\end{aligned}
\end{equation}
When approximating $W_{K}^{(l)}$ with LoRA during fine-tuning, i.e.,
$W_{K}^{(l)}=W_{K0}^{(l)}+B_{W_{K}^{(l)}}A_{W_{K}^{(l)}}$,
the NTK function can be derived as:
\begin{equation}
\label{eq:26}
\begin{aligned}
&K_{\text{attn;LoRA}}^{(l)}(x,x')\\
&=\nabla_{\theta_{\text{attn}}}x_{\text{attn}}^{(l)T}\cdot\nabla_{\theta_{\text{attn}}}x_{\text{attn}}^{(l)}\\
  &=\nabla_{\theta_{\text{attn}}^{(l)}}x_{\text{attn}}^{(l)T}\cdot\nabla_{\theta_{\text{attn}}^{(l)}}x_{\text{attn}}^{(l)}+\nabla_{\theta_{\text{attn}}^{(<l)}}x_{\text{attn}}^{(l)T}\cdot\nabla_{\theta_{\text{attn}}^{(<l)}}x_{\text{attn}}^{(l)}\\
  &=\nabla_{W_{Q}^{(l)}}x_{\text{attn}}^{(l)T}\cdot\nabla_{W_{Q}^{(l)}}x_{\text{attn}}^{(l)}+\nabla_{W_{K}^{(l)}}x_{\text{attn}}^{(l)T}\cdot\nabla_{W_{K}^{(l)}}x_{\text{attn}}^{(l)}+\nabla_{W_{V}^{(l)}}x_{\text{attn}}^{(l)T}\cdot\nabla_{W_{V}^{(l)}}x_{\text{attn}}^{(l)}+\nabla_{\theta_{\text{attn}}^{(<l)}}x_{\text{attn}}^{(l)T}\cdot\nabla_{\theta_{\text{attn}}^{(<l)}}x_{\text{attn}}^{(l)}\\
&=I_{d}\otimes
    (W_{V}^{(l)}x^{(l)})^{T}\dot{\mathbf{SM}}(\alpha_{a}^{(l)})(\frac{W_{Q}^{(l)}x^{(l)}}{\sqrt{d}})^{T}I_{d}\otimes
    x^{(l)T}A_{W_{K}^{(l)}}^{T}\cdot A_{W_{K}^{(l)}} x^{(l)'}\otimes I_{d}^{T}
    (\frac{W_{Q}^{(l)}x^{(l)'}}{\sqrt{d}})\dot{\mathbf{SM}}(\alpha_{a}^{(l)})(W_{V}^{(l)}x^{(l)'})\otimes
    I_{d}^{T}\\
  &\quad\quad\quad+\nabla_{W_{Q}^{(l)}}x_{\text{attn}}^{(l)T}\cdot\nabla_{W_{Q}^{(l)}}x_{\text{attn}}^{(l)}+\nabla_{W_{V}^{(l)}}x_{\text{attn}}^{(l)T}\cdot\nabla_{W_{V}^{(l)}}x_{\text{attn}}^{(l)}+\nabla_{\theta_{\text{attn}}^{(<l)}}x_{\text{attn}}^{(l)T}\cdot\nabla_{\theta_{\text{attn}}^{(<l)}}x_{\text{attn}}^{(l)}.
\end{aligned}
\end{equation}

Let $\mathbf{V}_{\text{attn}}^{(l)}(x)=A_{W_{K}^{(l)}} x^{(l)}\otimes I_{d}^{T}
    (\frac{W_{Q}^{(l)}x^{(l)}}{\sqrt{d}})\dot{\mathbf{SM}}(\alpha_{a}^{(l)})(W_{V}^{(l)}x^{(l)})\otimes
    I_{d}^{T}$ and $K_{\text{others}}^{(l)}=\nabla_{W_{Q}^{(l)}}x_{\text{attn}}^{(l)T}\cdot\nabla_{W_{Q}^{(l)}}x_{\text{attn}}^{(l)}+\nabla_{W_{V}^{(l)}}x_{\text{attn}}^{(l)T}\cdot\nabla_{W_{V}^{(l)}}x_{\text{attn}}^{(l)}+\nabla_{\theta_{\text{attn}}^{(<l)}}x_{\text{attn}}^{(l)T}\cdot\nabla_{\theta_{\text{attn}}^{(<l)}}x_{\text{attn}}^{(l)}$. Then, the NTK functions for full
    fine-tuning and LoRA can be simplified as:
   \begin{equation}
   \label{eq:27}
   \begin{aligned}
  K_{\text{attn;ff}}^{(l)}(x,x')&=\mathbf{V}_{\text{attn}}^{(l)T}(x)\cdot\mathbf{V}_{\text{attn}}^{(l)}(x')+K_{\text{others}}^{(l)}(x,x;')\\
  K_{\text{attn;LoRA}}^{(l)}(x,x')&=\mathbf{V}_{\text{attn}}^{(l)T}(x)A_{W_{K}^{(l)}}^{T}\cdot A_{W_{K}^{(l)}}\mathbf{V}_{\text{attn}}^{(l)}(x')+K_{\text{others}}^{(l)}(x,x;').
   \end{aligned}
 \end{equation} 
 Based on Theorem \ref{th:delta-nsd}, it is established that 
 $A_{W_{K}^{(l)}}^{T} A_{W_{K}^{(l)}} - I$ is \emph{negative
 semi-definite} under the specified conditions. This property directly
 leads to the validity of Corollary \ref{th:full-rank} and Theorem
 \ref{th:ib-leq} when comparing $K_{\text{attn;ff}}$ with
 $K_{\text{attn;LoRA}}$ shown in Equation \ref{eq:27}.

Moreover, by employing an analogous deduction procedure, it can be
demonstrated that these conclusions remain applicable when $W_{Q}^{(l)}$ or $W_{V}^{(l)}$ are approximated using LoRA.

\section{Explaining LoRA's Phenomenon through Our Analytical Framework}\label{sec:explain-lora-ntk}

Our analytical framework also provides novel insights into several
distinctive properties of LoRA that have previously lacked rigorous
explanation (shown in Section \ref{sec:related}).

\noindent
\textbf{Asymmetric Architecture of LoRA.}
Different from previous research~\cite{lora-asymmetry}, the inherent asymmetry of LoRA can be explicitly
captured by its NTK formulation. Specifically, the $A$ matrix plays a
direct and significant role in shaping the layer-wise kernel structure
of the NTK function. In contrast, the $B$
influences the NTK only indirectly through its impact on the
intermediate representations $y^{(l)}$.

\noindent
\textbf{Initialization Strategies for $A$ and $B$.}
\citet{lora-init} reveals that the initialization
strategies for matrices $A$ and $B$ are not \emph{interchangeable}, as swapping
their initialization schemes leads to performance
degradation. Our theoretical framework provides an elegant and
principled explanation for this phenomenon.
Specifically, initializing $A$ to
$\mathbf{0}$ renders the LoRA's NTK function (shown in Lemma \ref{lemma:ntk-lora}) \emph{degenerate}, effectively
reducing it to an identity transformation that preserves only the
input structure without meaningful feature extraction. Conversely,
initializing $B$ to $\mathbf{0}$ preserves the fundamental structure
of the NTK while allowing for effective adaptation during training.

\noindent
\textbf{The High Learning Rate Requirement of LoRA.} The averaged eigenvalue of $K_{\text{LoRA}}^{(l)}$'s kernel
matrix is typically smaller than that of $K_{\text{ff}}^{(l)}$,
demonstrating that the optimization step for LoRA under the same loss
is relatively smaller compared to full fine-tuning. Consequently, LoRA
introduces $\alpha$ to scale the learning rates according to the rank. When
the rank is small, a large $\alpha$ is recommended to mitigate the
negative impact of $r$ during fine-tuning.

\noindent
\textbf{Freezing $A$ Does not Affect LoRA's
  Fine-tuning Performance; in Some Cases, It is Even More Stable.} While \citet{lora-asymmetry} explains
this phenomenon with information theory, it can also be understood directly
through Lemma \ref{lemma:ntk-lora}. Specifically, $A$ appears explicitly
in $K_{\text{LoRA}}$. Given the second property of NTK
(Theorem \ref{th:ntk}), the $K_{\text{LoRA}}$'s kernel matrix should
keep constant during training. Forcibly freezing $A$ aligns with the ideal
conditions of the NTK regime in LoRA, which may explain why it
is beneficial.

\subsection{From R'{e}nyi Entropy to Shannon Entropy}\label{sec:h1}
In the standard definition of R\'{e}nyi entropy,
$H_\alpha=\frac{1}{1-\alpha}\log(\sum_{i=1}^{n_L}{P_i^\alpha})$, where
$0\leq P_i\leq 1$ and $\sum_{i=1}^{n_L}{P_i}=1$.

When $\alpha=1$, this expression becomes indeterminate (of the form $\frac{0}{0}$). However, in this case, the limit of $H_\alpha$ as $\alpha\rightarrow 1$ yields the Shannon entropy. Below is a brief derivation using L'Hopital's Rule:
\begin{equation}
\frac{d}{d\alpha}\log(\sum_{i=1}^{n_L}{P_i^\alpha})=\frac{\sum_{i=1}^{n_L}{P_i^\alpha\log{P_i}}}{\sum_{i=1}^{n_L}{P_i^\alpha}}, \frac{d}{d\alpha}{1-\alpha}=-1.
\end{equation}

Therefore,
\begin{equation}
\lim_{\alpha\rightarrow 1}{H_\alpha} = \lim_{\alpha\rightarrow 1}{\frac{\sum_{i=1}^{n_L}{P_i^\alpha\log{P_i}}}{\sum_{i=1}^{n_L}{P_i^\alpha}}}\cdot \frac{1}{-1}=-\sum_{i=1}^{n_L}{P_i\log{P_i}}.
\end{equation}

We actually utilize this Shannon entropy formula to demonstrate Figure 3.

\input{related}

\section{Supplemental Experiments}\label{sec:supp-exper}

\subsection{Evaluation with Additional Attack Strategies}\label{sec:add-attack}
We introduce four additional backdoor poisoning attacks in the NLP
setting: a clean-label backdoor poisoning attack (CL-BPA)~\cite{backdoor-poison}, an
instruction-level backdoor poisoning attack (IL-BPA)~\cite{iab}, a
multi-triggered stealthy backdoor attack (MT)~\cite{mt}, and a style-based
backdoor poisoning attack (S-BPA)~\cite{s-bpa}.

We adopt the same random seeds and experimental configurations when
assessing the resilience of LoRA under these additional attack
settings. The results are summarized below.

\begin{table}[htbp]
\centering
\caption{Performance comparison between FF and LoRA across different BPA attacks.}
\begin{tabular}{lcccc}
\toprule
\textbf{Model} & \textbf{Acc.} & \textbf{Pre.} & \textbf{Rec.} & \textbf{F1.} \\
\midrule
MT(FF)        & 82.91$\pm$6.77  & 75.96$\pm$7.71  & 98.64$\pm$0.98  & 85.66$\pm$4.75 \\
MT(LoRA)      & \textbf{89.14}$\pm$1.86  & 84.44$\pm$3.24  & 96.62$\pm$1.03  & \textbf{90.08}$\pm$1.42 \\
CL-BPA(FF)    & 91.78$\pm$0.47  & 89.47$\pm$0.91  & 95.04$\pm$0.39  & 92.17$\pm$0.41 \\
CL-BPA(LoRA)  & \textbf{92.39}$\pm$0.28  & 89.87$\pm$0.99  & \textbf{95.87}$\pm$0.72  & \textbf{92.77}$\pm$0.21 \\
IL-BPA(FF)    & 51.37$\pm$0.11  & 51.15$\pm$0.05  & 100.00$\pm$0.00 & 67.68$\pm$0.05 \\
IL-BPA(LoRA)  & \textbf{53.13}$\pm$2.35  & \textbf{52.09}$\pm$1.27  & 100.00$\pm$0.00 & \textbf{68.49}$\pm$1.09 \\
S-BPA(FF)     & 75.34$\pm$0.93  & 67.59$\pm$0.83  & 99.09$\pm$0.39  & 80.36$\pm$0.61 \\
S-BPA(LoRA)   & \textbf{85.51}$\pm$1.79  & \textbf{79.01}$\pm$2.33  & \textbf{97.52}$\pm$0.22  & \textbf{87.28}$\pm$1.34 \\
\bottomrule
\end{tabular}
\label{tab:main-results}
\end{table}

The experimental results indicate that LoRA demonstrates stronger
robustness than the full fine-tuning (FF) against a wide range of
mainstream backdoor attacks. This is consistent with both the
empirical evidence and the theoretical analysis presented in the
main paper.

\subsection{Evaluation on Other Initialization Strategies}\label{sec:more-init}

Besides of the default and most commonly used initialization strategy
(Kaiming Uniform) in LoRA, we evaluate two additional initialization
methods to examine the impact of their variances to LoRA's TTR. The
strategies include Xavier normal distribution-based initialization
(XNI)~\cite{xavier}, and Gaussian distribution-based initialization (GI).

\begin{table}[htbp]
\centering
\caption{Performance under different initialization strategies, variance scales, and poisoning rates.}
\begin{tabular}{l|ccc|cccc}
\toprule
\textbf{Init. Strategy} & \textbf{Scale of Variance} & \textbf{Poisoning Rate} & \textbf{Acc.} & \textbf{Pre.} & \textbf{Rec.} & \textbf{F1.} \\
\midrule
GI  & 0.33 & 0\%   & 93.00$\pm$0.49  & 92.40$\pm$1.97  & 94.05$\pm$1.47  & 93.19$\pm$0.37 \\
GI  & 1.0  & 0\%   & 92.98$\pm$0.60  & 92.45$\pm$2.24  & 93.96$\pm$1.92  & 93.16$\pm$0.51 \\
GI  & 2.0  & 0\%   & 93.07$\pm$0.63  & 92.88$\pm$2.21  & 93.64$\pm$1.73  & 93.23$\pm$0.55 \\
GI  & 0.33 & 0.15\%& 93.05$\pm$0.13  & 92.19$\pm$1.22  & 94.36$\pm$1.33  & 93.25$\pm$0.10 \\
GI  & 1.0  & 0.15\%& 92.79$\pm$0.33  & 92.22$\pm$0.19  & 93.82$\pm$1.75  & 92.99$\pm$0.25 \\
GI  & 2.0  & 0.15\%& 92.56$\pm$0.56  & 91.90$\pm$2.14  & 93.73$\pm$1.83  & 92.78$\pm$0.47 \\
  \hline
XNI & 0.33 & 0\%   & 93.18$\pm$0.44  & 92.25$\pm$1.48  & 94.59$\pm$1.08  & 93.39$\pm$0.35 \\
XNI & 1.0  & 0\%   & 92.91$\pm$0.34  & 92.16$\pm$1.76  & 94.14$\pm$1.65  & 93.11$\pm$0.29 \\
XNI & 2.0  & 0\%   & 93.11$\pm$0.34  & 92.35$\pm$1.34  & 94.32$\pm$1.04  & 93.31$\pm$0.27 \\
XNI & 0.33 & 0.15\%& 91.26$\pm$1.27  & 87.72$\pm$2.74  & 96.44$\pm$1.17  & 91.84$\pm$1.02 \\
XNI & 1.0  & 0.15\%& 89.97$\pm$2.82  & 85.55$\pm$4.52  & 97.02$\pm$1.13  & 90.85$\pm$2.18 \\
XNI & 2.0  & 0.15\%& 88.48$\pm$6.42  & 83.67$\pm$8.27  & 97.61$\pm$1.18  & 89.87$\pm$4.68 \\
\bottomrule
\end{tabular}
\label{tab:init-strategy-results}
\end{table}

The experimental results are generally consistent with those obtained
using the Kaiming Uniform initialization.

\subsection{LoRA's TTR on Generative Language Models}\label{sec:on-nlg}

Inspired by the BackdoorLLM~\cite{backdoorllm} benchmark, we evaluate the TTR of LoRA
against three backdoor poisoning attacks under two distinct attack
scenarios. The backdoor attacks include BadNet~\cite{badnet}, Sleeper Agent~\cite{sleeperagent}
(SA), and VPI~\cite{vpi}. The attack scenario is LLMs' jailbreaking, where a
backdoored LLM is expected to bypass safety filters (jailbreaking) to
answer certain queries when the input contains corresponding triggers.

We use the instruction-following dataset Alpaca~\cite{alpaca} as the supervised
fine-tuning (SFT) training set and choose LLaMA-3.2-3B as the model
backbone. We do not include LLaMA-3-8B due to GPU memory limitations
that prevent full fine-tuning on a single GPU. These experiments are
conducted on an Nvidia H100 GPU. The poisoning rate is set to 2\%.

The experimental results are shown below.

\begin{table}[htbp]
\centering
\caption{Attack success rate (ASR) under different backdoor methods on
generative language models.}
\begin{tabular}{lcc}
\toprule
\textbf{Backdoor Method} & \textbf{IsLoRA} & \textbf{ASR} \\
\midrule
BadNet & FF   & 90.91 \\
BadNet & LoRA & 84.85 \\
SA     & FF   & 92.93 \\
SA     & LoRA & 88.89 \\
VPI    & FF   & 86.87 \\
VPI    & LoRA & 84.85 \\
\bottomrule
\end{tabular}
\label{tab:asr-results}
\end{table}

We observe that the conclusions drawn from generative language models
are consistent with those from NLU models.

\subsection{Supplemental Results Corresponding to the Main Paper}

\begin{figure*}[h]
  \centering
  \includegraphics[width=0.92\linewidth]{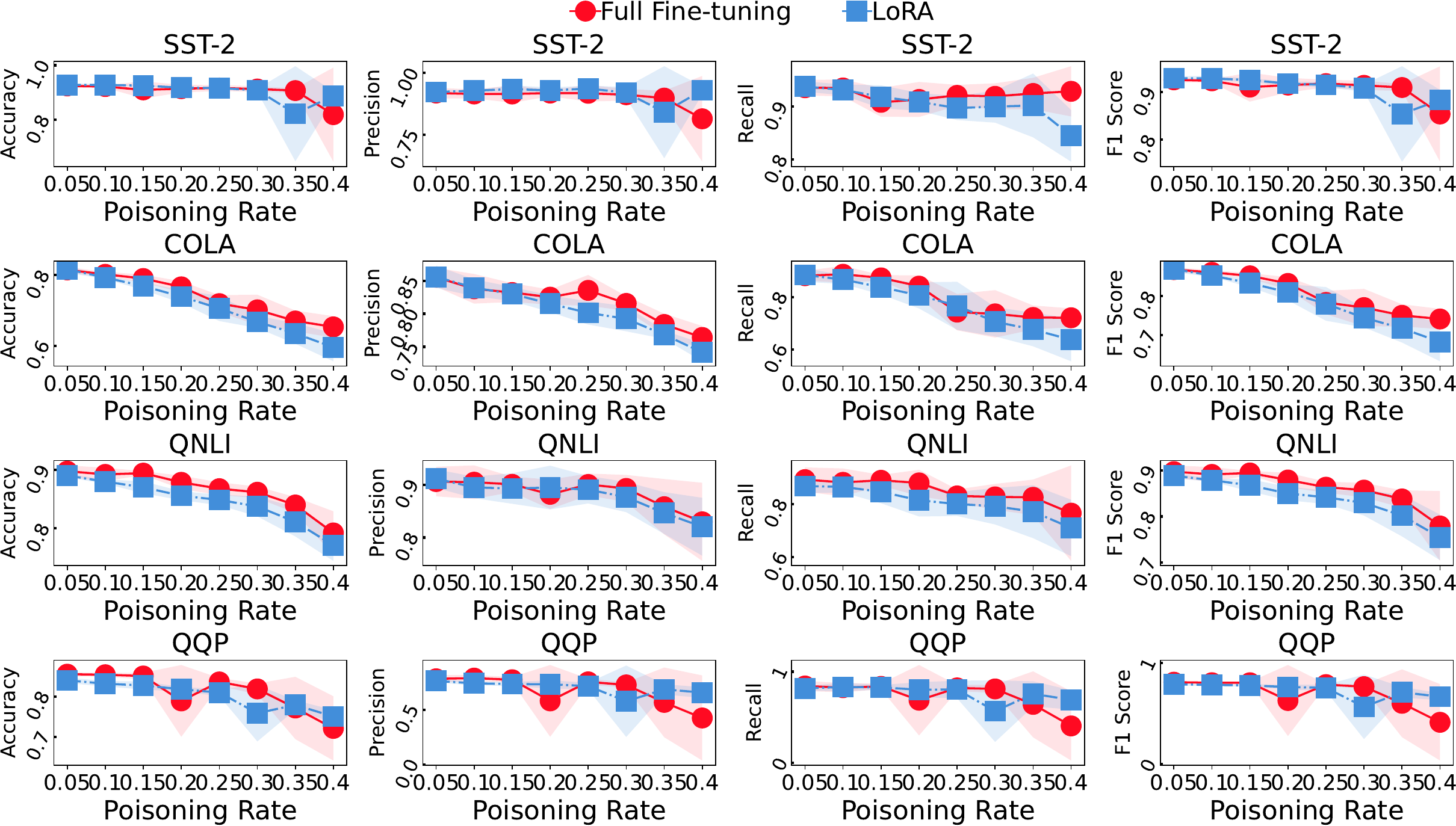}
  \caption{Performance comparison between full fine-tuning and LoRA
    under untargeted poisoning attacks with varying poisoning rates.}\label{fig:full-poison-pr}
\end{figure*}

\begin{figure*}[h]
  \centering
  \includegraphics[width=0.92\linewidth]{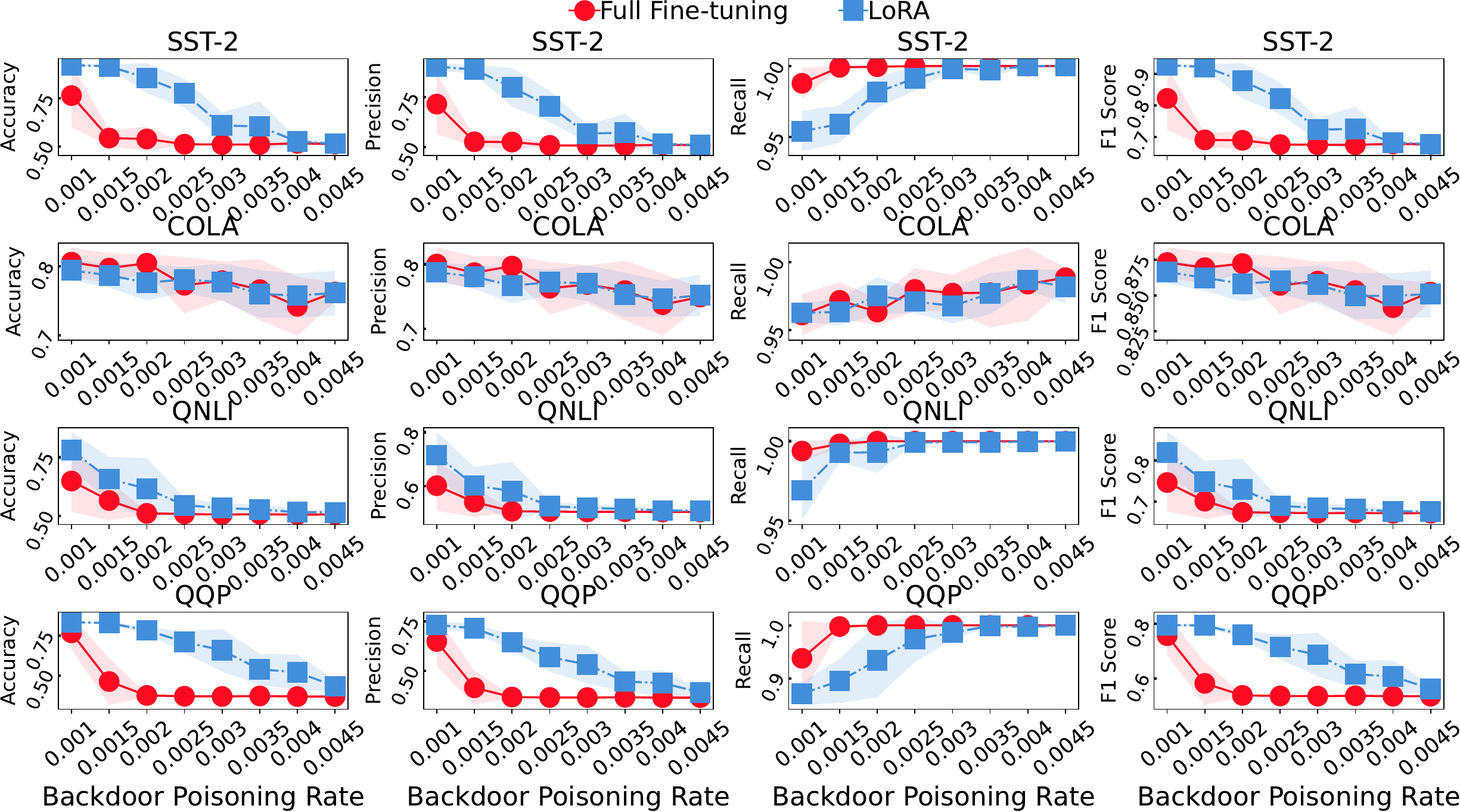}
  \caption{Performance comparison between full fine-tuning and LoRA
    under backdoor poisoning attacks with varying poisoning rates.}\label{fig:full-backdoor-pr}
\end{figure*}

\begin{figure*}[h]
  \centering
  \includegraphics[width=0.92\linewidth]{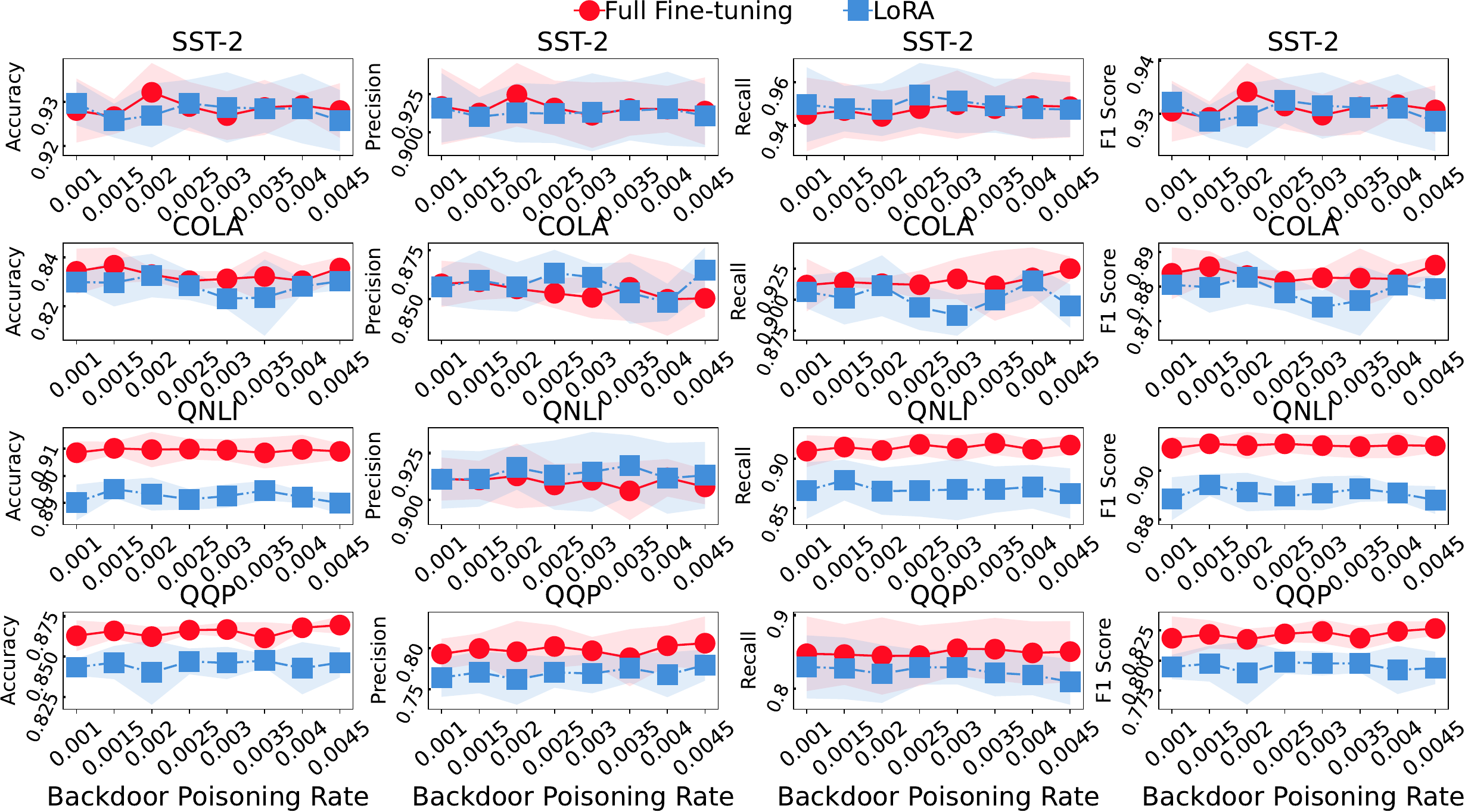}
  \caption{Performance comparison between full fine-tuning and LoRA
    under backdoor poisoning attacks with varying poisoning
    rates. Different from Figure \ref{fig:full-backdoor-pr}, we \textbf{do not
  employ triggers in the test samples}.}\label{fig:full-backdoor-pr-notrigger}
\end{figure*}

\begin{figure*}[h]
  \centering
  \includegraphics[width=0.92\linewidth]{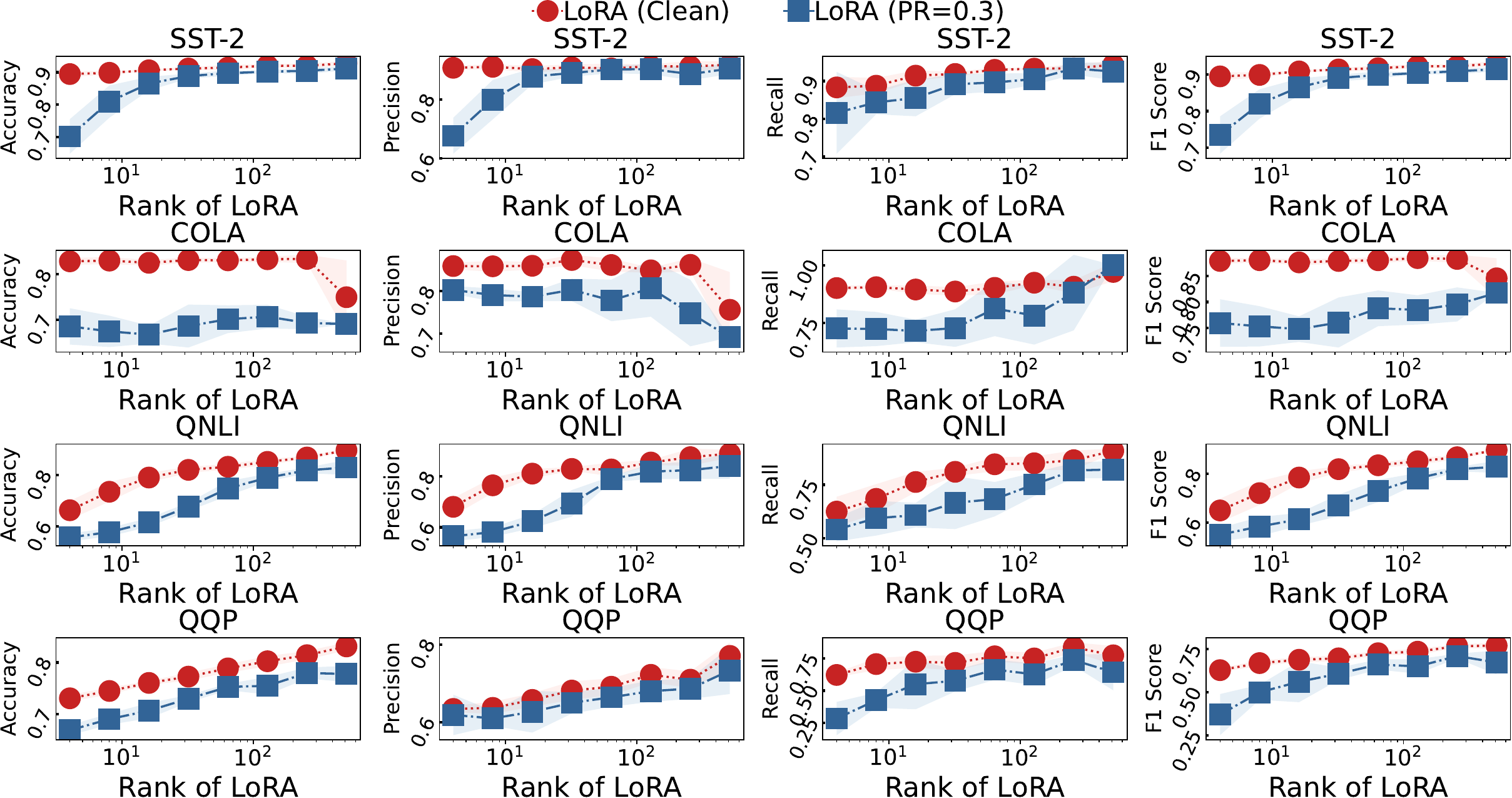}
  \caption{The effect of rank on LoRA's robustness under untargeted
    poisoning attacks.}\label{fig:full-poison-rank}
\end{figure*}

\begin{figure*}[h]
  \centering
  \includegraphics[width=0.92\linewidth]{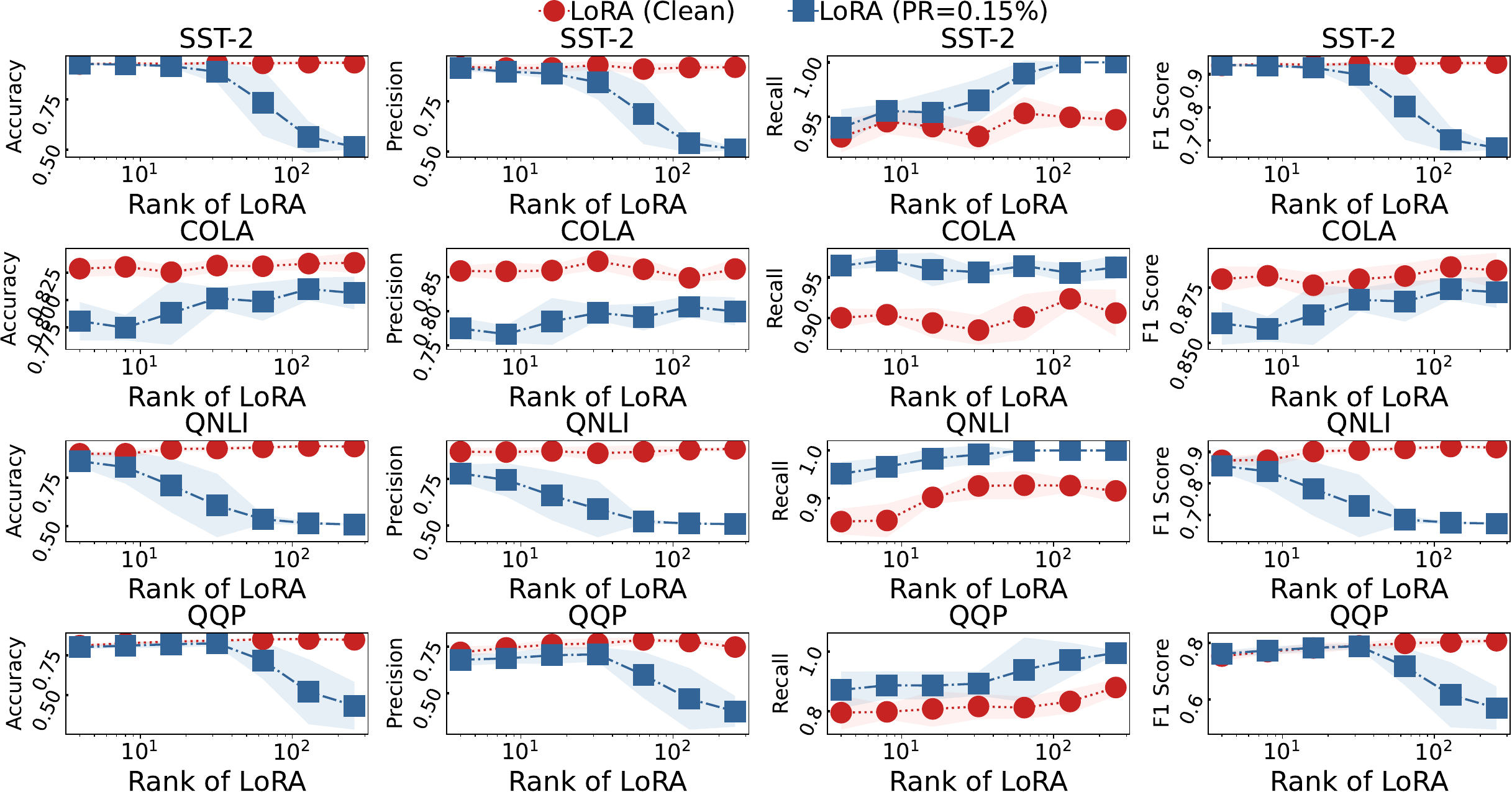}
  \caption{The effect of rank on LoRA's resistance under backdoor
    poisoning attacks.}\label{fig:full-backdoor-rank}
\end{figure*}

\begin{figure*}[h]
  \centering
  \includegraphics[width=0.97\linewidth]{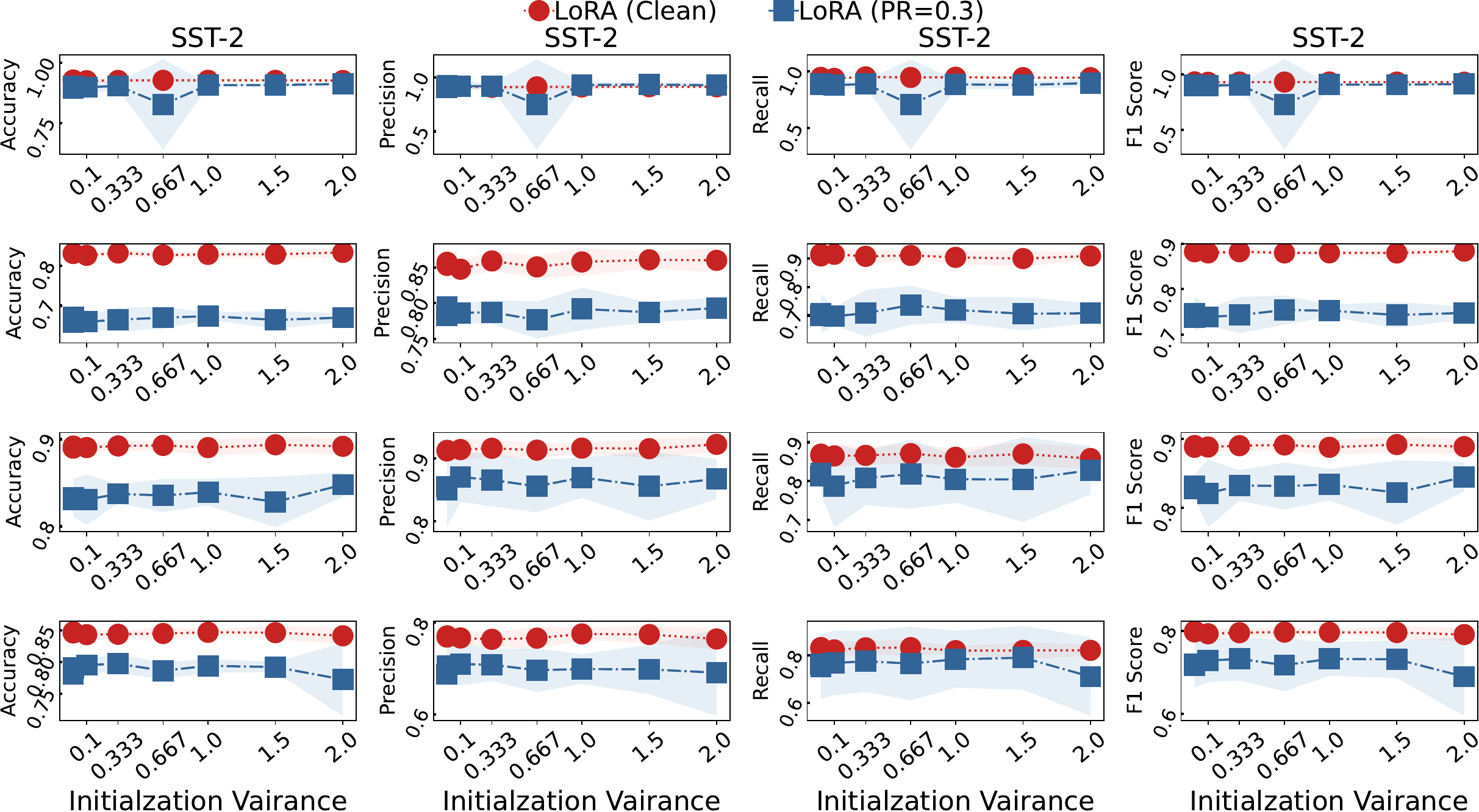}
  \caption{The effect of initialization variance on LoRA’s robustness under untargeted
    poisoning attacks.}\label{fig:full-poison-var}
\end{figure*}

\begin{figure*}[h]
  \centering
  \includegraphics[width=0.97\linewidth]{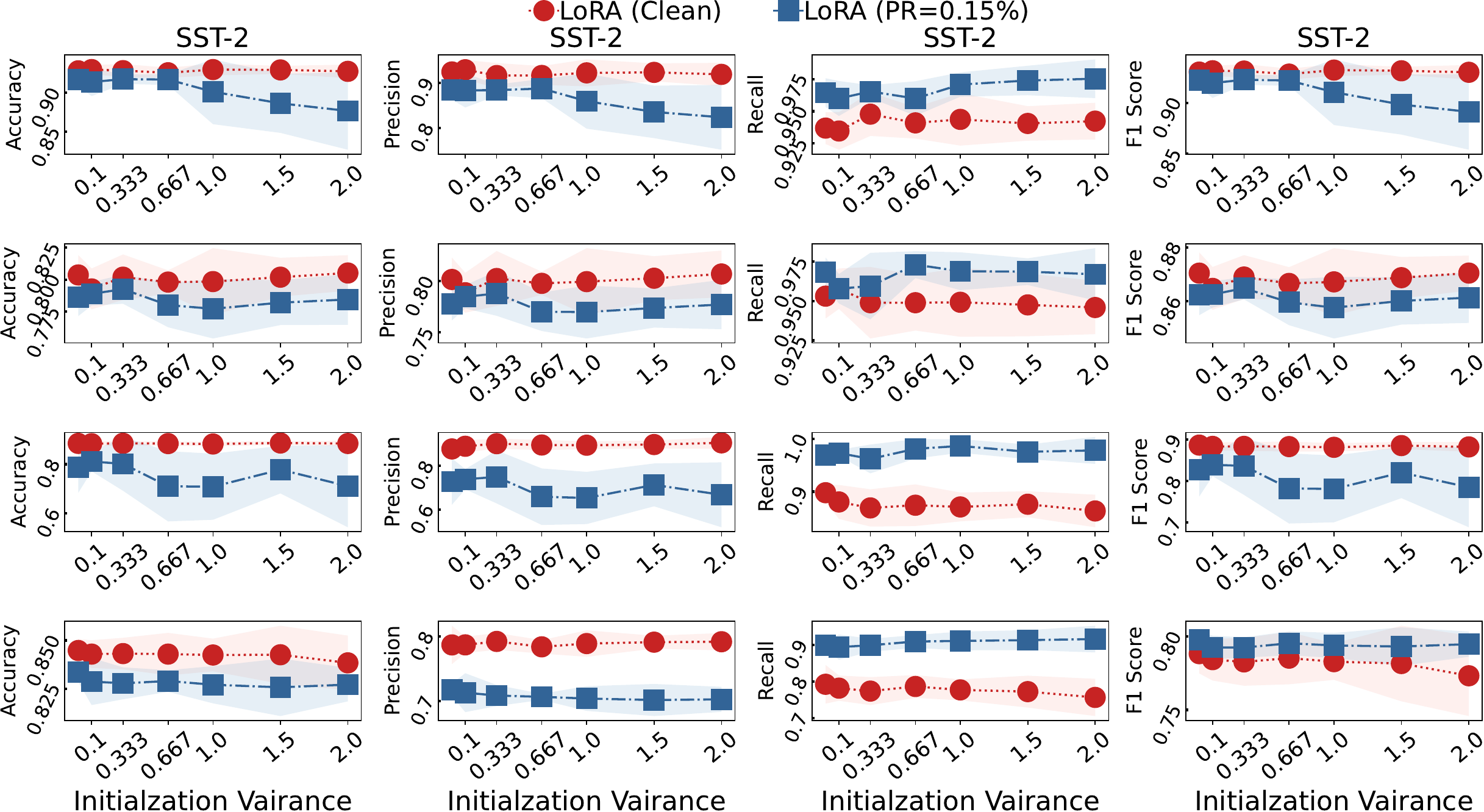}
  \caption{The effect of initialization variance on LoRA's resistance
    under backdoor poisoning attacks.}\label{fig:full-backdoor-var}
\end{figure*}


%% file: proofs.tex
\section{Proofs}

\subsection{Proofs of Theorem \ref{th:fisher-k}}\label{sec:proof-fisherk}

\begin{proof}  
  The Fisher information is formally defined by
\begin{equation}
  \label{eq:3}
 I_{\theta}= \mathbb{E}_{x\sim \mathcal{D}} \left[\nabla_{\theta}\mathcal{L}(x,\theta)^{T}\nabla_{\theta}\mathcal{L}(x,\theta)\right],
\end{equation}
where for neural networks, we can express the gradient as:
\begin{equation}
\label{eq:4}
\begin{aligned}
&\nabla_{\theta}\mathcal{L}(x,\theta)^{T}=\left(\nabla_{\theta}F_{\theta}\cdot \nabla_{F_{\theta}}\mathcal{L}(x,\theta)\right)^{T}.
\end{aligned}
\end{equation}
Through algebraic manipulation, we can derive:
\begin{equation}
\begin{aligned}
 I_{\theta}&= \mathbb{E}_{x\sim \mathcal{D}} \left[\nabla_{\theta}\mathcal{L}(x,\theta)^{T}\nabla_{\theta}\mathcal{L}(x,\theta)\right]\\
  &=\mathbb{E}_{x\sim \mathcal{D}}\left[\nabla_{F_{\theta}}\mathcal{L}(x,\theta)^{T}\nabla_{\theta}F_{\theta}^{T}\cdot\nabla_{\theta}F_{\theta}\nabla_{F_{\theta}}\mathcal{L}(x,\theta)\right]\\
  &=\mathbb{E}_{x\sim \mathcal{D}}\left[\nabla_{F_{\theta}}\mathcal{L}(x,\theta)^{T}K_{ntk}(x,x)\nabla_{F_{\theta}}\mathcal{L}(x,\theta)\right].
\end{aligned}
\end{equation}
\end{proof}

We can compute $\nabla_{F_{\theta}}\mathcal{L}(x,\theta)$ under
different loss functions.

\noindent
\textbf{Cross-Entropy Loss.}
For the cross-entropy loss function, we have:
\begin{equation}
\label{eq:7}
\begin{aligned}
\mathcal{L}(x,\theta)&=-\sum_{(x,y_{l})\sim \mathcal{D}}{\log Z[{y_{l}}]}\\
&=-\sum_{(x,y_{l})\sim \mathcal{D}}{\log \text{softmax}(F_{\theta}(x,\theta))},
\end{aligned}
\end{equation}
in which the corresponding gradient is:
\begin{equation}
\label{eq:8}
\begin{aligned}
\nabla_{F_{\theta}}\mathcal{L}(x,\theta)&=\nabla_{F_{\theta}}Z[y_{l}]\cdot\nabla_{Z[y_{l}]}\mathcal{L}(x,\theta)\\
&=Z[y_{l}](1-Z[y_{l}])\cdot \left(- \frac{1}{Z[y_{l}]}\right)\\
&=Z[y_{l}]-1\\
\end{aligned}
\end{equation}

This leads to the following relationship between the Fisher
information matrix and the NTK:
\begin{equation}
\label{eq:9}
\begin{aligned}
\mathcal{I}_{\theta}=\mathbb{E}_{x\sim \mathcal{D}}\left[(Z[y_{l}]-1)^{T}K_{ntk}(x,x)(Z[y_{l}]-1)\right].
\end{aligned}
\end{equation}

\noindent
\textbf{Mean Square Error Loss.}
The mean square error loss function is defined as:
\begin{equation}
\label{eq:10}
\begin{aligned}
\mathcal{L}(x,\theta)&=-\sum_{x,y_{l}\sim \mathcal{D}}{\frac{1}{2}(y_{l}-F_{\theta}(x,\theta))^{2}},
\end{aligned}
\end{equation}
where the gradient computation yields
\begin{equation}
\label{eq:12}
\nabla_{F_{\theta}}\mathcal{L}(x,\theta)=F_{\theta}(x,\theta)-y_{l}.
\end{equation}

\subsection{Deduction of the NTK Function}
\label{subsec:ntk-deduction}

\begin{proof}[The NTK function of full fine-tuning]
\begin{equation}
  \begin{aligned}
&K_{\text{ff}}^{(l,k)}(x,x')\\&=\nabla_{\theta}y^{(l,k)}(x)^{T}\nabla_{\theta}y^{(l,k)}(x')\\
  &=\nabla_{w\in~W^{(l)}}y^{(l,k)}(x)^{T}\nabla_{w\in~W^{(l)}}y^{(l,k)}(x')+\nabla_{\theta^{(<l)}}y^{(l,k)}(x)^{T}\nabla_{\theta^{(<l)}}y^{(l,k)}(x')\\
  &=\nabla_{w\in~W^{(l)}}y^{(l,k)}(x)^{T}\nabla_{w\in~W^{(l)}}y^{(l,k)}(x')\\&~~~~+\partial_{y_{a}^{(l-1)}}y^{(l,k)}(x)\partial_{y^{(l-1)}}y_{a}^{(l-1)}(x)\partial_{\theta^{(<l)}}y^{(l-1)}(x)\partial_{\theta^{(<l)}}y^{(l-1)}(x')^{T}\partial_{y^{(l-1)}}y_{a}^{(l-1)}(x')^{T}\partial_{y_{a}^{(l-1)}}y^{(l,k)}(x')^{T}\\
&=y_{a}^{(l-1)}(x)^{T}\cdot
  y_{a}^{(l-1)}(x')+\underbrace{W^{(l,k)}\dot{\sigma}(y^{(l-1)}(x))}_{\text{a
  scalar}}\underbrace{\phi^{(l-1)}(x)^{T}\phi^{(l-1)}(x')}_{\text{a
  scalar}}\underbrace{\dot{\sigma}(y^{(l-1)}(x'))^{T}W^{(l,k)~T}}_{\text{a
  scalar}}\\
&=y_{a}^{(l-1)}(x)^{T}\cdot
  y_{a}^{(l-1)}(x')+\underbrace{\phi^{(l-1)}(x)^{T}\phi^{(l-1)}(x')}_{K_{\text{ff}}^{(l-1,k)}(x,x')}\dot{\sigma}(y^{(l-1)}(x'))^{T}W^{(l,k)~T}W^{(l,k)}\dot{\sigma}(y^{(l-1)}(x)).\\
  \end{aligned}
\end{equation}

Based on the assumption of NTK that: \emph{i) $W^{(l)}$ is
initialized with the expectation of $0$ and variance of
$1/\sqrt{n_{l-1}}$; and ii) $n_{l-1}\rightarrow \infty$}, we can
derive $W^{(l,k)~T}W^{(l,k)}\rightarrow I_{n_{l-1}\times n_{l-1}}$, an
identity matrix, suggesting that the NTK of full fine-tuning can be
formalized as

\begin{equation}
  \begin{aligned}
 &K_{\text{ff}}^{(l,k)}(x,x')\\   
&=\underbrace{y_{a}^{(l-1)}(x)^{T}\cdot
  y_{a}^{(l-1)}(x')}_{\Sigma^{(l)}(x,x')}+\underbrace{\phi^{(l-1)}(x)^{T}\phi^{(l-1)}(x')}_{K_{\text{ff}}^{(l-1,k)}(x,x')}\dot{\sigma}(y^{(l-1)}(x'))^{T}\underbrace{W^{(l,k)~T}W^{(l,k)}}_{I_{n_{l-1}\times~n_{l-1}}}\dot{\sigma}(y^{(l-1)}(x))\\
&=\underbrace{y_{a}^{(l-1)}(x)^{T}\cdot
  y_{a}^{(l-1)}(x')}_{\Sigma^{(l)}(x,x')}+\underbrace{\phi^{(l-1)}(x)^{T}\phi^{(l-1)}(x')}_{K_{\text{ff}}^{(l-1,k)}(x,x')}\underbrace{\dot{\sigma}(y^{(l-1)}(x'))^{T}\dot{\sigma}(y^{(l-1)}(x))}_{\dot{\Sigma}^{(l)}(x,x')}\\
&=\Sigma^{(l)}(x,x')+K_{\text{ff}}^{(l-1,k)}(x,x')\dot{\Sigma}^{(l)}(x,x').\\
  \end{aligned}
\end{equation}
\end{proof}

\begin{proof}[The NTK function of LoRA]
\begin{equation}
  \begin{aligned}
&K_{\text{LoRA}}^{(l,k)}(x,x')\\&=\nabla_{\theta}y^{(l,k)}(x)^{T}\nabla_{\theta}y^{(l,k)}(x')\\
  &=\nabla_{B^{(l)}}y^{(l,k)}(x)^{T}\nabla_{B^{(l)}}y^{(l,k)}(x')+\nabla_{A^{(l)}}y^{(l,k)}(x)^{T}\nabla_{A^{(l)}}y^{(l,k)}(x')+\nabla_{\theta^{(<l)}}y^{(l,k)}(x)^{T}\nabla_{\theta^{(<l)}}y^{(l,k)}(x')\\
  &=\nabla_{B^{(l)}}y^{(l,k)}(x)^{T}\nabla_{B^{(l)}}y^{(l,k)}(x')+\partial_{z^{(l)}}y^{(l,k)}(x)\partial_{A^{(l)}}z^{(l)}(x)\partial_{A^{(l)}}z^{(l)}(x')^{T}\partial_{z^{(l)}}y^{(l,k)}(x')^{T}\\&~~~~+\partial_{y_{a}^{(l-1)}}y^{(l,k)}(x)\partial_{y^{(l-1)}}y_{a}^{(l-1)}(x)\partial_{\theta^{(<l)}}y^{(l,k)}(x)\partial_{\theta^{(<l)}}y^{(l-1)}(x')^{T}\partial_{y^{(l-1)}}y_{a}^{(l-1)}(x')^{T}\partial_{y_{a}^{(l-1)}}y^{(l,k)}(x')^{T}\\
&=z^{(l-1)}(x)^{T}\cdot z^{(l-1)}(x')+B^{(l,k)}\cdot~I_{r}\otimes\sigma(y^{(l-1)}(x))^{T}\sigma(y^{(l-1)}(x'))\otimes~I_{r}^{T}\cdot~B^{(l,k)~{T}}\\&~~~~+(W_{0}^{(l,k)}+B^{(l,k)}A^{(l,k)})\dot{\sigma}(y^{(l-1)}(x))\phi^{(l-1)}(x)^{T}\phi^{(l-1)}(x')\dot{\sigma}(y^{(l-1)}(x'))^{T}(W_{0}^{(l,k)}+B^{(l,k)}A^{(l,k)})^{T}\\
&=y_{a}^{(l-1)}(x)^{T}A^{(l)~T}A^{(l)}y_{a}^{(l-1)}(x')+B^{(l,k)}\cdot~I_{r}\otimes\underbrace{\sigma(y^{(l-1)}(x))^{T}\sigma(y^{(l-1)}(x'))}_{\text{a~scalar}}\otimes~I_{r}^{T}\cdot~B^{(l,k)~{T}}\\&~~~~+\underbrace{(W_{0}^{(l,k)}+B^{(l,k)}A^{(l,k)})\dot{\sigma}(y^{(l-1)}(x))}_{\text{a~scalar}}\underbrace{\phi^{(l-1)}(x)^{T}\phi^{(l-1)}(x')}_{\text{a~scalar}}\underbrace{\dot{\sigma}(y^{(l-1)}(x'))^{T}(W_{0}^{(l,k)}+B^{(l,k)}A^{(l,k)})^{T}}_{\text{a~scalar}}\\
&=y_{a}^{(l-1)}(x)^{T}A^{(l)~T}A^{(l)}y_{a}^{(l-1)}(x')+\underbrace{\sigma(y^{(l-1)}(x))^{T}\sigma(y^{(l-1)}(x'))}_{\dot{\Sigma}^{(l)}(x,x')}\underbrace{B^{(l,k)}\cdot~B^{(l,k)~{T}}}_{\text{a~scalar}}\\&~~~~+\underbrace{\phi^{(l-1)}(x)^{T}\phi^{(l-1)}(x')}_{K_{\text{LoRA}}^{(l-1,k)}(x,x')}\dot{\sigma}(y^{(l-1)}(x'))^{T}(W_{0}^{(l,k)}+B^{(l,k)}A^{(l,k)})^{T}(W_{0}^{(l,k)}+B^{(l,k)}A^{(l,k)})\dot{\sigma}(y^{(l-1)}(x)).\\
  \end{aligned}
\end{equation}

In LoRA, the matrix $B^{(l)}$ is initialized to the zero matrix $\mathbf{0}$, indicating that
$W^{(l)}_{\text{LoRA}}=W^{(l)}_{0}+B^{(l)}A^{(l)}=W^{(l)}_{0}+\mathbf{0}\cdot
A^{(l)}=W^{(l)}_{0}$ at the begin of training. Similar to the NTK of
full fine-tuned models, we can also demonstrate that 
$W_{\text{LoRA}}^{(l)~T}W_{\text{LoRA}}^{(l)}\rightarrow
I_{n_{l-1}\times n_{l-1}}$.

Therefore, we have

\begin{equation}
  \begin{aligned}
&K_{\text{LoRA}}^{(l,k)}(x,x')\\
&=y_{a}^{(l-1)}(x)^{T}A^{(l)~T}A^{(l)}y_{a}^{(l-1)}(x')+\underbrace{\sigma(y^{(l-1)}(x))^{T}\sigma(y^{(l-1)}(x'))}_{\dot{\Sigma}^{(l)}(x,x')}\underbrace{B^{(l,k)}\cdot~B^{(l,k)~{T}}}_{\text{a~scalar}}\\&~~~~+\underbrace{\phi^{(l-1)}(x)^{T}\phi^{(l-1)}(x')}_{K_{\text{LoRA}}^{(l-1,k)}(x,x')}\dot{\sigma}(y^{(l-1)}(x'))^{T}(W_{0}^{(l,k)}+B^{(l,k)}A^{(l,k)})^{T}(W_{0}^{(l,k)}+B^{(l,k)}A^{(l,k)})\dot{\sigma}(y^{(l-1)}(x)).\\
&=\underbrace{y_{a}^{(l-1)}(x)^{T}A^{(l)~T}A^{(l)}y_{a}^{(l-1)}(x')}_{\Sigma^{(l)}_{\text{LoRA}}(x,x')}+\underbrace{\sigma(y^{(l-1)}(x))^{T}\sigma(y^{(l-1)}(x'))}_{\dot{\Sigma}^{(l)}(x,x')}\underbrace{\mathbf{0}_{1\times~n_{l}}\cdot\mathbf{0}_{n_{l}}}_{0}\\&~~~~+\underbrace{\phi^{(l-1)}(x)^{T}\phi^{(l-1)}(x')}_{K_{\text{LoRA}}^{(l-1,k)}(x,x')}\dot{\sigma}(y^{(l-1)}(x'))^{T}\underbrace{(W_{0}^{(l,k)}+B^{(l,k)}A^{(l,k)})^{T}(W_{0}^{(l,k)}+B^{(l,k)}A^{(l,k)})}_{I_{n_{l-1}\times n_{l-1}}}\dot{\sigma}(y^{(l-1)}(x)).\\
&=\underbrace{y_{a}^{(l-1)}(x)^{T}A^{(l)~T}A^{(l)}y_{a}^{(l-1)}(x')}_{\Sigma^{(l)}_{\text{LoRA}}(x,x')}+\underbrace{\phi^{(l-1)}(x)^{T}\phi^{(l-1)}(x')}_{K_{\text{LoRA}}^{(l-1,k)}(x,x')}\underbrace{\dot{\sigma}(y^{(l-1)}(x'))^{T}\dot{\sigma}(y^{(l-1)}(x))}_{\dot{\Sigma}^{(l)}_{\text{LoRA}}(x,x')}\\
&=\Sigma^{(l)}_{\text{LoRA}}(x,x')+K_{\text{LoRA}}^{(l-1)}(x,x')\dot{\Sigma}_{LoRA}^{(l)}(x,x')\\
&=\Sigma^{(l)}_{\text{LoRA}}(x,x')+K_{\text{LoRA}}^{(l-1)}(x,x')\cdot\dot{\Sigma}^{(l)}(x,x').\\
  \end{aligned}
\end{equation}
\end{proof}
The aforementioned theoretical analysis primarily focuses on
artificial neural networks (ANNs) initialized with random
weights. However, in more practical scenarios, particularly in
continuous fine-tuning settings, empirical observations demonstrate
that the network dynamics remain within the Neural Tangent Kernel
(NTK) regime. This phenomenon is further supported by experimental
evidence presented in Section \ref{sec:related}.

\subsection{Proofs of Theorem \ref{th:rel}}
\label{subsec:proof-threl}
\begin{proof}
Leveraging the two properties of NTK, we establish that NTK functions keep
constant during the training procedure. Consequently, our analysis
focuses on deriving the relationship between $K_{\text{LoRA}}^{(l)}$
and $K_{\text{ff}}^{(l)}$ at the initialization stage.

In LoRA, the weight matrix is typically initialized as $A^{(l)}\sim \mathcal{P}(0,\sigma^{2})$
and $B^{(l+1)}=0$, where $\mathcal{P}(0,\sigma^{2})$ denotes a probability
distribution with the expectation $0$ and variance $\sigma^{2}$. This
class of distributions encompasses common initialization schemes such as
Gaussian distribution, Kaiming distribution, and so on. At
initialization, we observe the following equivalence:
\begin{equation}
W^{(l)}_{\text{ff}}=W^{(l)}_{0}=W^{(l)}_{0}+\mathbf{0}=W^{(l)}_{0}+B^{(l)}A^{(l)}=W^{(l)}_{\text{LoRA}}.
\end{equation}

Consequently, we can derive that
$\dot{\Sigma}^{(l)}_{\text{LoRA}}(x,x')=\dot{\Sigma}^{(l)}_{\text{ff}}(x,x')=\dot{\Sigma}^{(l)}$.

Building upon Theorem \ref{th:rel}, which states that the first $l-1$ layers
maintain identical configurations between full fine-tuning and LoRA, we know that
$K_{\text{LoRA}}^{(l-1)}=K_{\text{ff}}^{(l-1)}=K^{(l-1)}$ and
$y^{(l-1)}_{\text{LoRA}}(x)=y^{(l-1)}_{\text{ff}}(x)$ and
$y^{(l-1)}_{\text{LoRA}}(x')=y^{(l-1)}_{\text{ff}}(x')$. The NTK functions
for both LoRA and full fine-tuning can be formatted as 
\begin{equation}
\label{eq:11}
\begin{aligned}
&K_{\text{ff}}^{(l,k)}(x,x')=K^{(l-1,k)}\dot{\Sigma}^{(l)}+ \Sigma^{(l)}_{\text{ff}}(x,x')=K^{(l-1,k)}\dot{\Sigma}^{(l)}+ \mathbf{\sigma}(y^{(l-1)}(x))^{T}\cdot \mathbf{\sigma}(y^{(l-1)}(x'))\\
&K_{\text{LoRA}}^{(l,k)}(x,x')=K^{(l-1,k)}\dot{\Sigma}^{(l)}+ \Sigma^{(l)}_{\text{LoRA}}(x,x')=K^{(l-1,k)}\dot{\Sigma}^{(l)}+\sigma(y^{(l-1)}(x))^{T}A^{(l)~T}\cdot A^{(l)}\sigma(y^{(l-1)}(x')).
\end{aligned}
\end{equation}

Through algebraic manipulation, we derive their fundamental relationship:
\begin{equation}
\label{eq:rel}
K_{\text{LoRA}}^{(l,k)}=K_{\text{ff}}^{(l,k)}+\Delta_{r}^{(l,k)},
\end{equation}
where the residual term is defined as:
\begin{equation*}
  \begin{aligned}
\Delta_{r}^{(l,k)}&= [ \mathbf{\sigma} ( y^{(l-1)}(x) )
]^{T}(A^{(l)~T}A^{(l)}-I_{n_{l-1}\times n_{l-1}}) [ \mathbf{\sigma} ( y^{(l-1)}(x') ) ].
  \end{aligned}
\end{equation*}
\end{proof}

\subsection{Proofs of Theorem \ref{th:delta-nsd}}\label{subsec:proof-rel-ineq}

\begin{proof}
Leveraging the fundamental properties of matrix rank, we have:
\begin{equation}
\mathbf{rank}(A^{(l)T}A^{(l)})\leq \mathbf{rank}(A^{(l)}) \leq r.
\end{equation}
The condition $\mathbf{rank}(A^{(l)T}A^{(l)})\leq r$ indicates that there at most exist
$n-r$ nonzero eigenvalues in $A^{(l)T}A^{(l)}$. 

Given the initialization conditions $\mathbb{E}[A^{(l)}]=\mathbf{0}$ and $\mathbf{Var}(A^{(l)})=\sigma^{2}$,
we derive the following expectation for
any column index $p=\{0,1,2,...,n-1\}$:
\begin{equation}\label{eq:25}
\mathbb{E}[A_{\cdot,p}^{^{(l)}T}A_{\cdot,p}^{(l)}]=\mathbb{E}[\sum\limits_{q=1}^{r}A_{q,p}^{(l)}\cdot~A_{q,p}^{(l)}]=r\sigma^{2}.
\end{equation}
The expected trace of $A^{(l)T}A$ can be formalized by
\begin{equation}
\mathbb{E}[\mathbf{tr}(A^{(l)T}A^{(l)})]=\mathbb{E}[\sum\limits_{p=1}^{n_{l-1}}\sum\limits_{q=1}^{r}A_{q,p}^{(l)}\cdot~A_{q,p}^{(l)}]=n_{l-1}\cdot r\sigma^{2}.
\end{equation}

Considering the eigenvalue distribution of $A^{(l)T}A^{(l)}$, we note that $n_{l-1}-r$ of them are $0$,
while the remaining $r$ eigenvalues are identically distributed with
the same expectation. Thus, the expected value of the rest $r$ eigenvalues is
$\mathbb{E}_{\lambda\in \text{Eigen}\{A^{(l)T}A^{(l)}\}}[\lambda_{i}]=n_{l-1}\cdot \sigma^{2}$.
When $n_{l-1}\rightarrow \infty$, all nonzero eigenvalues of $A^{(l)T}A$
converge to $n_{l-1}\sigma^{2}$, where if
$\sigma^{2}<\frac{1}{n_{l-1}}$, they are
smaller than $n_{l-1}\cdot 1/n_{l-1}=1$. In conclusion, we proof that all eigenvalues of $A^{(l)T}A$ are
smaller than 1 if $\sigma^{2}<\frac{1}{n_{l-1}}$. Consequently,
$A^{(l)T}A-I$ exhibits exclusively non-positive eigenvalues,
proving that $A^{(l)T}A-I$ is negative semi-definite when
$r<n_{l-1}$ and $\sigma^{2}\leq \frac{1}{n_{l-1}}$.
\end{proof}

\begin{proof}[Proof of Corollary \ref{th:full-rank}]
Building upon our theoretical analysis, we establish the proof of Corollary \ref{th:full-rank}.

When $r=n_{l-1}$, the matrix $A\in \mathbb{R}^{n_{l-1}\times n_{l-1}}$
becomes square. The expectation of its Gram matrix entries is given by:
and 
\begin{equation}
\label{eq:13}
\mathbb{E}[(A^{(l)T}A^{(l)})_{p,q}]=\mathbb{E}[\sum\limits_{u=1}^{n_{l-1}}A_{u,p}^{(l)}\cdot~A_{q,u}^{(l)}].
\end{equation}
Under the weight initialization scheme, we have
$\mathbb{E}[A_{u,v}]=0$ and $\mathbf{Var}(A_{u,v})=\sigma^{2}$ for all
$u,v \in \{1,...,n_{l-1}\}$, with independent entries. This leads to
the following cases:
\begin{itemize}
\item For off-diagonal entries ($q\neq p$):
\begin{equation}
\label{eq:20}
\mathbb{E}[A_{u,p}\cdot A_{q,u}]=\mathbb{E}[A_{u,p}]\cdot\mathbb{E}[A_{q,u}]=0.
\end{equation}
\item For diagonal entries ($p=q$), analogous to Equation \ref{eq:25}:
\begin{equation}
\label{eq:21}
\mathbb{E}[(A^{(l)T}A^{(l)})_{p,q}]=\mathbb{E}[\sum\limits_{u=1}^{n_{l-1}}A_{u,p}^{(l)}\cdot~A_{q,u}^{(l)}]=n_{l-1}\cdot
\sigma^{2}.
\end{equation}
\end{itemize}
When the initialization variance satisfies $\sigma^{2}=1/n_{l-1}$, the
diagonal entries simplify to $\mathbb{E}[(A^{(l)T}A^{(l)})_{p,q}]=1$.

Consequently, when $n_{l-1}\rightarrow \infty$, we conclude that
$A^{(l)T}A^{(l)}\rightarrow I$.
\end{proof}

\subsection{Proofs of Theorem \ref{th:ib-leq}}\label{sec:proof-ib-leq}

\begin{proof}
  Base on Theorem \ref{th:delta-nsd}, we know that when $r\leq
  n_{l-1}$ and $\sigma^{2}\leq 1/n_{l-1}$, the kernel matrix
  $M_{\Delta}^{l}$ is negative semi-definite. This implies that all $M_{\Delta}^{l}$'s eigenvalues $\lambda_{\Delta}^{l}\leq 0$.

  Then $\forall~\nabla_{F_{\theta}}\mathcal{L}(x,\theta)
  \in \mathbb{R}^{n_{L}}$, we derive the following inequalities:
  \begin{equation}
  \label{eq:14}
  \begin{aligned}
  &\nabla_{F_{\theta}}\mathcal{L}(x,\theta)^{T}\Delta_{r}(x,x)\nabla_{F_{\theta}}\mathcal{L}(x,\theta)\leq 0\\
    &\Rightarrow
    \nabla_{F_{\theta}}\mathcal{L}(x,\theta)^{T}K_{\text{LoRA}}(x,x)\nabla_{F_{\theta}}\mathcal{L}(x,\theta) \leq\nabla_{F_{\theta}}\mathcal{L}(x,\theta)^{T}K_{\text{FF}}(x,x)\nabla_{F_{\theta}}\mathcal{L}(x,\theta)\\
    &\Rightarrow
      \mathcal{I}_{\theta~\text{LoRA}}\leq\mathcal{I}_{\theta~\text{FF}}\\
  \end{aligned}
\end{equation}
Then $\forall~\lambda_{\text{LoRA}} \in
\text{Eigen}(\mathcal{I}_{\theta~\text{LoRA}}^{I})$ and $\forall~\lambda_{\text{FF}}^{I} \in
\text{Eigen}(\mathcal{I}_{\theta~\text{FF}})$, we have
\begin{equation}
\label{eq:15}
\lambda_{\text{LoRA}}^{I}\leq\lambda_{\text{FF}}^{I}.
\end{equation}

This eigenvalue relationship leads to the following important results:
\begin{equation}
\label{eq:16}
\begin{aligned}
&\frac{1}{2}\sum_{\lambda_{\theta~\text{LoRA}}^{I}}{\lambda_{\theta~\text{LoRA}}^{I}}\leq\frac{1}{2}\sum_{\lambda_{\theta~\text{ff}}^{I}}{\lambda_{\theta~\text{ff}}^{I}}\\
&\Rightarrow\mathbf{IB}_{\text{LoRA}}\leq\mathbf{IB}_{\text{ff}}
\end{aligned}
\end{equation}
and
\begin{equation}
\label{eq:17}
\begin{aligned}
&\frac{1}{1-\alpha}\log\left(\sum_{i=1}^{n_{L}}{\lambda_{\theta~\text{LoRA}}^{I}}\right)\leq\frac{1}{1-\alpha}\log\left(\sum_{i=1}^{n_{L}}{\lambda_{\theta~\text{FF}}^{I}}\right)\\
&\Rightarrow H_{\alpha \text{LoRA}}\leq H_{\alpha \text{ff}}.
\end{aligned}
\end{equation}
\end{proof}

\subsection{Proofs of Theorems beyond the OOLD Assumption}\label{sec:proof-by-oold}

\subsubsection{Proofs of Theorem \ref{th:delta-nsd} beyond the OOLD Assumption}
\begin{proof}
 Let $K_{\text{ff}}^{(l,k)'}$ and $K_{\text{LoRA}}^{(l,k)'}$ denote the NTKs of FF and LoRA beyond the OOLD assumption.
  From Equation \ref{eq:ffntk} and Equation \ref{eq:kntk}, we
  derive the difference of initialized NTK functions as follows:
  \begin{equation}
  \label{eq:18}
  \begin{aligned}
  \Delta^{(1,k)'}&=K_{\text{LoRA}}^{(1,k)'}-K_{\text{ff}}^{(1,k)'}=0;\\
  \Delta^{(2,k)'}&=K_{\text{LoRA}}^{(2,k)'}-K_{\text{ff}}^{(2,k)'}\\
    &=(K_{\text{LoRA}}^{(1,k)}-K_{\text{ff}}^{(1,k)})\dot{\Sigma}^{(2)}+\sigma(y^{(1)}(x))^{T}A^{(2)~T}A^{(2)}\sigma(y^{(1)}(x))-\sigma(y^{(1)}(x))^{T}\sigma(y^{(1)}(x))\\
    &=\sigma(y^{(1)}(x))^{T}(A^{(2)~T}A^{(2)}-I)\sigma(y^{(1)}(x));\\
  \Delta^{(l,k)'}&=K_{\text{LoRA}}^{(l,k)'}-K_{\text{ff}}^{(l,k)'}\\
    &=(K_{\text{LoRA}}^{(l-1,k)'}-K_{\text{ff}}^{(l-1,k)'})\dot{\Sigma}^{(l)}+\sigma(y^{(l-1)}(x))^{T}A^{(l)~T}A^{(l)}\sigma(y^{(l-1)}(x))-\sigma(y^{(l-1)}(x))^{T}\sigma(y^{(l-1)}(x))\\
    &=\Delta^{(l-1,k)'}\dot{\Sigma}^{(l)}+\sigma(y^{(l-1)}(x))^{T}A^{(l)~T}A^{(l)}\sigma(y^{(l-1)}(x))-\sigma(y^{(l-1)}(x))^{T}\sigma(y^{(l-1)}(x))\\
    &=\Delta^{(l-1,k)'}\dot{\Sigma}^{(l)}+\sigma(y^{(l-1)}(x))^{T}(A^{(l)~T}A^{(l)}-I)\sigma(y^{(l-1)}(x))\\
    &=\Delta^{(l-1,k)'}\dot{\Sigma}^{(l)}+\Delta_{r}^{(l)}.\\
  \end{aligned}
\end{equation}

By Theorem \ref{th:delta-nsd}, the matrix
$A^{(l)~T}A^{(l)}-I$ is negative semi-definite when
$\sigma_{a}^{2}<1/n_{l-1}$ and $r\leq n_{l-1}$. Consequently,
$\forall~y^{(1)(x)}\in \mathbb{R}^{n_{1}}$, we have
$\Delta^{(2,k)'}\leq 0$ and $\Delta^{(l,k)}\leq 0$. Moreover, since $\forall
~y^{(l)}\in \mathbb{R}^{n_{l}}, \dot{\sigma}(y^{(l)})\geq
0$, it follows that $\Delta^{(l,k)'}\geq 0$ for $l=3,...,L$. In
conclusion, $\Delta^{(l,k)'}\geq 0$ holds for $l=1,...,L$.
\end{proof}

\subsubsection{Proofs of Corollary \ref{th:full-rank} beyond the OOLD Assumption}
\begin{proof}
  When $\sigma_{a}^{2}=1/n_{l-1}$ and $r= n_{l-1}$, the matrix
  $A^{(l)~T}A^{(l)}-I\rightarrow\mathbf{0}$ when $n_{l}\rightarrow \infty$.

  Given $A^{(l)~T}A^{(l)}-I=\mathbf{0}$, we obtain
  $\Delta^{(2,k)'}=0$, and
  $\Delta^{(l,k)'}=\Delta^{(l-1,k)'}\dot{\Sigma}^{(l)}+\Delta^{(l)}_{r}$
  for $l=3,...,L$, where
  \begin{equation}
  \label{eq:19}
  \begin{aligned}
  &\Delta^{(l,k)'}=\Delta^{(l-1,k)'}\dot{\Sigma}^{(l)}+\Delta^{(l)}_{r}\\
  &=\mathbf{0}\dot{\Sigma}^{(l)}+\mathbf{0}=\mathbf{0}.
  \end{aligned}
\end{equation}
Thus, $\Delta^{(l,k)'}=\mathbf{0}$ for all layers $l$.
\end{proof}

\subsubsection{Proof of Theorem \ref{th:ib-leq} beyond the OOLD Assumption.}

The proof follows a similar structure to the proof provided in Appendix \ref{sec:proof-ib-leq}.

\begin{proof}
  Base on Theorem \ref{th:delta-nsd}, when $r\leq
  n_{l-1}$ and $\sigma^{2}\leq 1/n_{l-1}$, the kernel matrix
  $M_{\Delta}^{l}$ is negative semi-definite, implying that all of the
  $M_{\Delta}^{l}$'s eigenvalues $\lambda_{\Delta}^{l}\leq 0$.

  $\forall~\nabla_{F_{\theta}}\mathcal{L}(x,\theta)
  \in \mathbb{R}^{n_{L}}$, we derive the following inequalities:
  \begin{equation}
  \label{eq:14}
  \begin{aligned}
  &\nabla_{F_{\theta}}\mathcal{L}(x,\theta)^{T}\Delta_{r}(x,x)\nabla_{F_{\theta}}\mathcal{L}(x,\theta)\leq 0\\
    &\Rightarrow
    \nabla_{F_{\theta}}\mathcal{L}(x,\theta)^{T}K_{\text{LoRA}}(x,x)\nabla_{F_{\theta}}\mathcal{L}(x,\theta) \leq\nabla_{F_{\theta}}\mathcal{L}(x,\theta)^{T}K_{\text{FF}}(x,x)\nabla_{F_{\theta}}\mathcal{L}(x,\theta)\\
    &\Rightarrow
      \mathcal{I}_{\theta~\text{LoRA}}\leq\mathcal{I}_{\theta~\text{FF}}\\
  \end{aligned}
\end{equation}
This implies that $\forall~\lambda_{\text{LoRA}} \in
\text{Eigen}(\mathcal{I}_{\theta~\text{LoRA}}^{I})$ and $\forall~\lambda_{\text{FF}}^{I} \in
\text{Eigen}(\mathcal{I}_{\theta~\text{FF}})$, we have
\begin{equation}
\label{eq:15}
\lambda_{\text{LoRA}}^{I}\leq\lambda_{\text{FF}}^{I}.
\end{equation}

Consequently, we establish the following results:
\begin{equation}
\label{eq:16}
\begin{aligned}
&\frac{1}{2}\sum_{\lambda_{\theta~\text{LoRA}}^{I}}{\lambda_{\theta~\text{LoRA}}^{I}}\leq\frac{1}{2}\sum_{\lambda_{\theta~\text{ff}}^{I}}{\lambda_{\theta~\text{ff}}^{I}}\\
&\Rightarrow\mathbf{IB}_{\text{LoRA}}\leq\mathbf{IB}_{\text{ff}}
\end{aligned}
\end{equation}
and
\begin{equation}
\label{eq:17}
\begin{aligned}
&\frac{1}{1-\alpha}\log\left(\sum_{i=1}^{n_{L}}{\lambda_{\theta~\text{LoRA}}^{I}}\right)\leq\frac{1}{1-\alpha}\log\left(\sum_{i=1}^{n_{L}}{\lambda_{\theta~\text{FF}}^{I}}\right)\\
&\Rightarrow H_{\alpha \text{LoRA}}\leq H_{\alpha \text{ff}}.
\end{aligned}
\end{equation}
\end{proof}

%% file: related.tex
\section{Supplemental Related Works}\label{sec:related}

\noindent
\textbf{Theoretical Analysis on LoRA.} LoRA is inspired by the
\emph{intrinsic low-rank hypothesis}~\cite{low-d-intrinsic},
which assumes that the learnable matrices in neural networks are
typically over-parameterized relative to their actual required
dimension. Building on this hypothesis, several works have delved into
the underlying mechanisms of LoRA. For instance,
\citet{lora-expressive} explores the expressive capacity of LoRA and
proved that a neural network model fine-tuned with LoRA can fit any
smaller target models, once the rank of LoRA exceeds a threshold
determined by the architectural properties of the two neural
networks. This finding establishes a low bound of LoRA's rank to
achieve ideal convergence.
Some research explains LoRA by analyzing its structural
characteristics. For instance, ~\citet{lora-transformer} study the
impact of the attention mechanism of Transformer architectures. A
noteworthy contribution comes from ~\citet{lora-asymmetry}, who
investigate the \emph{asymmetry} between the two submatrices (as defined in Equation \ref{eq:lora}) in
LoRA. By freezing one submatrice while observing the behavior of the
other, they reveal distinct roles in LoRA, i.e., matrix $A$ functions
as a feature extractor, while $B$ maps these features to the desired
output. Based on these findings, they propose freezing $A$ and
fine-tuning only $B$, achieving comparable performance and
better generalization capabilities. 
In terms of LoRA's learning dynamics, the neural tangent
kernel~\cite{ntk} has been employed as a theoretical
framework~\cite{ntk-local,ft-kernel}. Specifically, \citet{ft-kernel}
empirically demonstrate that parameter-efficient fine-tuning (PEFT),
including LoRA, stays within a NTK regime. They then indicate that LoRA's fine-tuning is nearly equivalent to full
fine-tuning (FF). Besides, \citet{ntk-local} proposed that a rank
$r>\sqrt{N_{tr}}$ with training samples number $N_{tr}$, is sufficient
to eliminate spurious local minima during training, thereby enabling
effective generalization in few-shot learning tasks.

While these studies offer valuable insights into the underlying
mechanisms of LoRA, certain aspects, especially the potential security
concerns when replacing full fine-tuning with LoRA, remain
insufficiently explored. To address this gap,  we
delve into the training procedure of LoRA, and analyze their potential
security vulnerabilities in the paper.



\noindent
\textbf{Kernel Views of Neural Networks.}\label{sec:related-kernel}
A kernel function
$k(x,x'):\mathbb{R}^{d}\times\mathbb{R}^{d}\rightarrow \mathbb{R}$
is typically defined as a mapping from two vectors $x$ and $x'$ to their
correlation score $k(x,x')$. This score can be interpreted as the
\textbf{inner product} of the two vectors under an unknown
high-dimensional transformation function. \citet{gp} were the first to
reveal that the feed-forward procedure of a neural network can be seen
as a \emph{Gaussian process (GP)} when the network width approaches
infinity. They prove that the kernel function associated with such a
GP is determined by the architecture and parameters of the neural
network.
Building on the same infinite-width assumption, \citet{ntk}
demonstrated that the parameter updates of a neural network can be
characterized by a special kernel function, termed as \emph{neural
  tangent kernel (NTK)}, with the form given by
\begin{equation}
\label{eq:ntk2}
K_{ntk}(x,x')=\nabla_{\theta}F(x;\theta)^{T}\nabla_{\theta}F(x';\theta),
\end{equation}
where $F(x;\theta)$ denotes a neural network's output with
parameters $\theta$.

\citet{ntk} demonstrated that, as the width of the neural
network approaches infinity, the NTK exhibits the following two key
properties:
\begin{enumerate}
\item\label{item:ntk} The NTK converges to a
  \emph{deterministic} limiting kernel that depends only on three
  factors: \emph{i)} the variance of the parameter initialization,
  \emph{ii)} the neural network structure, and \emph{iii)} the selection of activation functions;
\item The NTK keeps \emph{constant} through out each training step $t$.
\end{enumerate}
These properties greatly simplify the theoretical analysis for a neural
network's training process.

While the \emph{infinite} width assumption is somewhat impractical for
neural networks, recent studies~\cite{ntk-appro,mea-ntk} have aimed to extend
NTK theory to more realistic settings, such as using Taylor
expansions. As an empirical observation, \citet{ft-kernel}
suggests that \textbf{prompt-based fine-tuning of language models
  still operates within the NTK regime}.
Inspired by the two properties of NTK and this observation, we adopt
NTK as a framework to model the
training-time robustness of LoRA compared to full fine-tuning (FF).


